\documentclass[11pt]{article}
\usepackage[square,numbers,sort]{natbib}
\usepackage[margin=1in]{geometry}
\usepackage[utf8]{inputenc}
\usepackage{graphicx}
\usepackage{bbm}
\usepackage{subfig}
\usepackage{adjustbox}
\usepackage{booktabs}
\usepackage{multirow}
\def\Real{\mathop{\mathbb{R}}\nolimits}

\def\argmin{\mathop{\rm argmin}\nolimits}

\newcommand{\bx}{\boldsymbol{x}}
\newcommand{\by}{\boldsymbol{y}}

\newcommand{\bmu}{\boldsymbol{\mu}}
\newcommand{\bA}{\boldsymbol{A}}
\newcommand{\bB}{\boldsymbol{B}}

\newcommand{\bI}{\boldsymbol{I}}

\newcommand{\bL}{\boldsymbol{L}}

\newcommand{\bQ}{\boldsymbol{Q}}

\newcommand{\bS}{\boldsymbol{S}}

\newcommand{\bV}{\boldsymbol{V}}

\newcommand{\bX}{\boldsymbol{X}}

\newcommand{\bZ}{\boldsymbol{Z}}

\newcommand{\bsigma}{\boldsymbol{\sigma}}

\newcommand{\bPi}{\boldsymbol{\Pi}}

\newcommand{\E}{\mathbb{E}}
\newcommand{\I}{\mathcal{I}}
\newcommand{\J}{\mathcal{J}}
\newcommand{\cL}{\mathcal{L}}
\newcommand{\cO}{\mathcal{O}}
\newcommand{\sP}{\mathbb{P}}
\newcommand{\one}{\mathbbm{1}}
\newcommand{\Hilbert}{\mathcal{H}}

\usepackage{amsmath, amssymb,amsthm}
\usepackage{algorithm}
\usepackage{algpseudocode}
\usepackage{array}
\usepackage{array, makecell}
\usepackage{longtable}
\usepackage{booktabs}
\usepackage{authblk}
\usepackage{hyperref}
\usepackage[dvipsnames]{xcolor}
\definecolor{royalblue}{rgb}{0.0, 0.22, 0.66}
\definecolor{RoyalBlue}{cmyk}{1, 0.90, 0, 0}
\hypersetup{
    citecolor=RoyalBlue,
    colorlinks=true,
    linkcolor=RoyalBlue,
    filecolor=magenta,      
    urlcolor=cyan,
}

\usepackage{tikz}
\usetikzlibrary{arrows.meta,
                chains,
                positioning,
                shapes.geometric,
                automata,
                positioning
                }
\makeatletter
\def\mathcolor#1#{\@mathcolor{#1}}
\def\@mathcolor#1#2#3{%
  \protect\leavevmode
  \begingroup
    \color#1{#2}#3%
  \endgroup
}
\makeatother

\newtheorem{theorem}{Theorem}
\newtheorem{corollary}{Corollary}

\usepackage{mathrsfs}

\theoremstyle{remark}

\newtheorem{ass}{A \hspace{-6pt}}

\begin{document}

\title{Robust Principal Component Analysis: A Median of Means Approach}

 \author[1]{Debolina Paul}
\author[2]{Saptarshi Chakraborty\thanks{Correspondence to: \texttt{saptarshic@berkeley.edu}.}}

 \author[3,4]{Swagatam Das}
 
 \affil[1]{Department of Statistics, Stanford University}
 \affil[2]{Department of Statistics, University of California, Berkeley}
 \affil[3]{Electronics and Communication Sciences Unit, Indian Statistical Institute, Kolkata, India}
\affil[4]{Institute for Advancing Intelligence (IAI), TCG CREST, Kolkata India}

\maketitle
\begin{abstract}
Principal Component Analysis (PCA) is a fundamental tool for data visualization, denoising, and dimensionality reduction. It is widely popular in Statistics, Machine Learning, Computer Vision, and related fields. However, PCA is well-known to fall prey to outliers and often fails to detect the true underlying low-dimensional structure within the dataset. Following the Median of Means (MoM) philosophy, recent supervised learning methods have shown great success in dealing with outlying observations without much compromise to their large sample theoretical properties. This paper proposes a PCA procedure based on the MoM principle. Called the \textbf{M}edian of \textbf{M}eans \textbf{P}rincipal \textbf{C}omponent \textbf{A}nalysis (MoMPCA), the proposed method is not only computationally appealing but also achieves optimal convergence rates under minimal assumptions. In particular, we explore the non-asymptotic error bounds of the obtained solution via the aid of the Rademacher complexities while granting absolutely no assumption on the outlying observations. The derived concentration results are not dependent on the dimension because the analysis is conducted in a separable Hilbert space, and the results only depend on the fourth moment of the underlying distribution in the corresponding norm. The proposal's efficacy is also thoroughly showcased through simulations and real data applications.
\end{abstract}




%

\section{Introduction}
\label{sec:introduction}


%
%
%
%
Principal component analysis (PCA) \cite{pearson1901liii, wold1987principal} is perhaps the most well-known statistical method for linear dimensionality reduction \cite{hastie2009elements}. Given a set of (mean-centered) points $\mathcal{X}=\{\bx_1,\dots,\bx_N\} \subset \Real^p$, the PCA computes a small number of orthonormal basis vectors, which characterize most of the variability within the data cloud. 
Mathematically, to find a $d$-dimensional ($d \le p)$ representation of $\mathcal{X}$, one projects these $N$ points on to a $d$-dimensional subspace as $\{\bQ\bx_1,\dots,\bQ\bx_N\}$. where, $\bQ$ is a $p \times p$ projection matrix of rank $d$. Denoting the set of all $p \times p$ projection matrices of rank $d$ as $\mathcal{Q}_d$, the PCA minimizes the following objective:
\begin{equation}\label{e1}
   \frac{1}{N}\sum_{i=1}^N \|\bx_i - \bQ \bx_i\|^2 = \frac{1}{N}\sum_{i=1}^N \bx_i^\top (\bI - \bQ) \bx_i, 
\end{equation}
where $\bQ    \in \mathcal{Q}_d$.

Though efficient in many real-life scenarios, PCA is well documented to suffer from many disadvantages such as information loss, poor performance when the data lies in a non-affine manifold, poor performance for high-dimensional data, etc. Researchers continue to tackle these issues through modifications to the original versions, such as probabilistic PCA \cite{tipping1999probabilistic}, kernel PCA \cite{scholkopf1997kernel}, sparse PCA \cite{zou2006sparse}, robust PCA \cite{candes2011robust,NEURIPS2019_73f104c9}, and so on. Among many disadvantages of PCA, like information loss and poor interpretation, one of the major concerns is that it is well known to fail in the presence of a single outlying observation. Its fragility against severely corrupted data points often puts its reliability at risk. Popular approaches for making the PCA more robust to these outlying and corrupted observations attempt to represent the original data matrix ${\bx}$ as a sum of a low-rank matrix $\bL$ and a sparse matrix $\bS$ \cite{wright2009,candes2011robust,chandrasekaran2011rank,kang2015robust,pmlr-v48-chiang16, 8701558, 8818654}. However, many methods mentioned earlier only allow recovery guarantees under very restrictive assumptions. For example, \cite{candes2011robust} restricts the structure of $\bS$ to a Bernoulli model.  


Robust PCA has recently gained popularity, primarily due to the rise of Big Data. These big datasets often contain outliers and hence require delicate handling, usually done by implementing robust statistics against outlying observations. Recently, there has been much interest in the robust dimensionality reduction community on proposing efficient frameworks for PCA \cite{tang2011robust,vaswani2018robust,wang2017angle}. There has also been an array of work based on geometric median-based approaches \cite{cardot2013efficient,cardot2017fast,fritz2012comparison,cohen2016geometric}, which focuses on a robust estimation of the covariance matrix and finding the eigendecomposition of the same to find the proper subspace to project. Although efficient in practice, many techniques mentioned above do not come with finite-sample theoretical guarantees. Those with asymptotic rates often assume that the data points are independent and identically distributed, which does not hold when the data is contaminated with outliers.



To bridge this methodological gap, the Median of Means (MoM) literature provides a promising and attractive framework to adapt PCA to become outlier robust, as well as help us preserve our theoretical understanding of finite-sample error bounds. 
As opposed to the classical Vapnik-Chervonenkis Empirical Risk Minimization (ERM) \cite{vapnik2013nature}, the  MoM philosophy provides a more robust framework for efficiently finding estimates of the actual underlying parameter. Although MoM estimators have been in the literature for quite a long time, it has recently been introduced to the Machine learning community \cite{lugosi2019regularization, lecue2020robust, bartlett2002model, lecue2020class}. Besides being insensitive to outliers, the MoM estimators also possess a solid theoretical backbone comprising exponential concentration results under the mild restriction of finite variance \cite{lugosi2019regularization, lecue2020robust, bartlett2002model,lerasle2019lecture,laforgue2019medians}. Recently, several near-optimal results were established from this perspective concerning regression \cite{mathieu2021excess,lugosi2019regularization}, bandits \cite{bubeck2013bandits}, mean estimation  \cite{minsker2018uniform}, clustering \cite{klochkov2020robust,brunet2020k,paul2021uniform}, classification \cite{lecue2020robust1}, and optimal transport \cite{pmlr-v130-staerman21a}. 

In this paper, we develop a Principal Component Analysis (PCA) in the framework of the MoM principle. The proposal is not only computationally efficient but also theoretically appealing. It is well known that the theoretical understanding of many classical Empirical Risk Minimization (ERM) \cite{vapnik2013nature} such as PCA hinges on the assumptions that the data should be independently and identically distributed (i.i.d.) and have sub-gaussian behavior. However, real data that may be corrupted with outliers do not offer us the luxury to make such simplifying assumptions. Towards our theoretical investigation, we assume that the dataset can be split into two categories: the set of inliers ($\I$) and the set of outliers ($\cO$), i.e., $\{1,\dots,N\}=\I \cup \cO$. The data points in $\I$ are assumed to be independently distributed according to the distribution $P$, which has a finite fourth moment. We make no assumptions on the points in $\cO$; thus, allowing them to be dependent, unboundedly large, having distributions that are entirely dissimilar to $P$, allowing them to be heavy-tailed, etc. Our theoretical analysis hinges on the application of Rademacher complexity \cite{bartlett2002rademacher}, and symmetrization arguments \cite{vapnik2013nature,devroye2013probabilistic}. From a theoretical viewpoint, we emphasize that our analyses are derived under more general and interpretable conditions compared to the literature \cite{lecue2020robustcl}. To further generalize the proposed setup, we conduct all our theoretical analyses in a separable Hilbert space. This allows us to derive \textit{dimension-free bounds} that only depend on the data distribution through the fourth moment of the inliers. We quickly recover the data results in a finite-dimensional vector space showing that the derived rates match the state-of-the-art \cite{paul2021uniform,lecue2020robust}.

The main contributions of this article can be summarized as follows:

\begin{itemize}

    \item This paper proposes a simple yet efficient framework for robust PCA under the median of means
paradigm. Apart from being practically efficient, the method comes with a strong theoretical backbone of finite-sample error rates under the mild assumption of the existence of a finite fourth moment of the underlying data distribution. Such consistency guarantees and error rates that are derived subsequently, ensure that the results are reliable and accurate under mild assumptions.

\item Furthermore, the derived generalization bounds are dimension-free meaning that the error rate is also valid for infinite-dimensional Hilbert spaces.

\item We emphasize that as opposed to many prominent works \cite{zhang2015analysis,candes2011robust}, we only assume that the number of outliers is
$o(N)$, which is the natural definition of outliers. We do not make any assumptions on the
distribution of the outliers allowing them to be dependent, unboundedly large, having distributions that are entirely dissimilar to the inlier distribution, etc.

\item The detailed experimental analysis via simulation and real-life datasets on background modeling in video and anomaly detection demonstrates the efficacy of MoMPCA compared to the state-of-the-art in different experimental settings, indicating that the
the proposed method is highly effective in practice.

\end{itemize}

The rest of the paper is organized as follows. In section \ref{formulation}, we formulate the MoMPCA, followed by a detailed theoretical analysis under minimal and interpretable assumptions in section \ref{theory}. The experimental results are discussed in section \ref{experments}, followed by concluding remarks in section \ref{conclusion}.

\textbf{Notations:} Before we proceed further, we discuss a few notations used in this paper. Vectors are dented with bold lower-case letters, and matrices are denoted with bold uppercase letters. $\langle \bA , \bB \rangle = \text{trace} (\bA^\top \bB)$ denotes the Frobenious inner product between two matrices. $\|\bA\| =\sqrt{\langle \bA , \bA \rangle}$ denotes the Frobenious norm of the matrix $\bA$. $\sP(E)$ denotes the probability of the event $E$ and $\E_{\bZ}(\cdot)$ denotes the expectation with respect to the random vector $\bZ$. A random variable $\sigma$ is said to be Rademacher if it takes values in $\{+1,-1\}$ with equal probability. $[N]$ denotes the set $\{1,\dots,N\}$, for any $N \ge 0$ and $2^A$ denotes the power set of the $A$. For any two sets $A$ and $B$, $A^B$ denotes the set of all functions from $B$ to $A$. $\mu_k=\int\|\bx\|_2^k dP$ denote the $k$-th moment of $P$. $\mathcal{Q}_d$ denotes the set of all $p\times p$ real projection matrices of rank $d$. $one(\cdot)$ denotes the indicator function.
\subsection{Related Works}
This section discusses some of the works related to robust clustering. The most noted work in this direction is arguably the work by \cite{candes2011robust}, which appeals to the philosophy of writing the data matrix as a sum of a low rank and a sparse matrix and minimizing the subsequent approximation error. \cite{tang2011robust} extended this idea by developing a convex program for a low-rank and block-sparse matrix decomposition. A survey on these papers can be found in \cite{vaswani2018robust}. \cite{liu2014scatter} developed a new subspace selection method called angle linear discriminant embedding (ALDE) for dimensionality reduction in a supervised learning setting. Angle PCA proposed by \cite{wang2017angle} develops an iterative algorithm to minimize an $\ell_2$ norm-based error and maximize the summation of the ratio between the variance and the reconstruction error to preserve rotational invariance while tackling outliers. \cite{liu2020angular} further extended these ideas to develop angular embedding for conducting robust PCA. \cite{hauberg2014grassmann,hauberg2015scalable} introduced Grassman averages to express dimensionality reduction as an average of the subspaces spanned by the data. The authors further exploit its properties to develop a robust Grassman average as a form of robust PCA. \cite{bouwmans2018applications} showed the application of robust PCA in image and video processing. Other than the median of the means PCA method, there are other types of robust PCA methods as well, such as the optimal mean methods \cite{gao2017angle, gao2020enhanced} and the avoid mean calculation \cite{liao2018robust}. There has also been an array of work based on geometric median-based approaches \cite{cardot2013efficient,cardot2017fast,fritz2012comparison,cohen2016geometric}, which focuses on a robust estimation of the covariance matrix and finding the eigendecomposition of the same to find the proper subspace to project.

\subsection{Motivation and a Proof of Concept Result}
Robust PCA is commonly formulated by assuming that the data matrix can be decomposed into a low-rank signal component and a noise component. However, this framework is limited by the assumption that the effect of outliers is mild and cannot grow unboundedly large. To address this issue, the authors propose using the median of means framework, which is commonly used in robust mean estimation, to view traditional PCA as an empirical risk minimization problem. Because the objective cannot be solved using the eigen-decomposition trick, the authors propose an alternative method called projected Adagrad. The authors note that the most challenging aspect of the project was extending the theoretical results to an infinite-dimensional Hilbert space and deriving dimension-free bounds under only the existence of a finite fourth moment of the data distribution. This result not only provides meaningful conclusions with minimal assumptions but also paves the way for future research on robust kernel-based methods.

As a motivating example, in Fig.~\ref{fig:motivate}, we show the results of classical PCA and MoMPCA on a toy dataset (available at \href{https://github.com/SaptarshiC98/MOMPCA}{https://github.com/SaptarshiC98/MOMPCA}). The dataset has $1000$ observations, out of which $10$ of them, i.e., $1\%$ are outliers. The inliers are generated from a $2$-dimensional Gaussian random variable with variance $10$ in the first dimension, $1$ in the second dimension, and covariance $8$. The outliers are generated from another $2$-dimensional Gaussian random variable with mean vector $(15,50)$, variance $5$ in each dimension with zero covariance between them. Finally, we run the vanilla PCA (shown in red) and the proposed MoMPCA (shown in blue) and plot their first principal components. From Fig.~\ref{fig:motivate}, it is clear that a mere $1$\% outlying observations are enough to render the vanilla PCA ineffective, whereas, MoMPCA correctly identifies the direction of maximum variation. In terms of excess risk (details provided in section \ref{recover}), we calculate that the excess risk for vanilla PCA is $20.6850$ where, in comparison, the extra risk for MoMPCA is only $0.3104$, proving its efficacy.

\begin{figure}[ht]
    \centering
    \includegraphics[height=0.25\textwidth,width=0.48\textwidth]{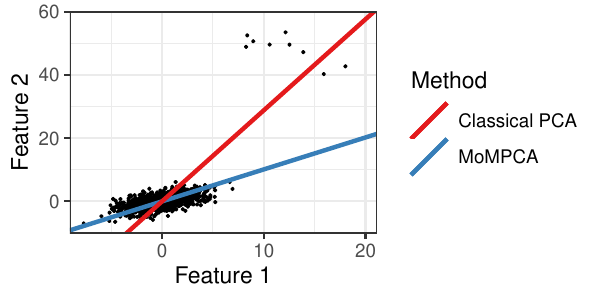}
    \hspace{-30pt}
    \caption{The first principal component found out by PCA and MoM PCA for the motivating example. Even in the presence of only $1\%$ outlying observations, the classical PCA can render spurious results while MoM PCA finds the direction of maximum variation efficiently.}
    \label{fig:motivate}
\end{figure}


\section{PCA with Median of Means}
\label{formulation}
\subsection{Problem Setup}
We will follow the same notations as described in section \ref{sec:introduction}. One should note that the PCA objective function \eqref{e1} can be written as $\int \xi_V(\bx) P_N(d\bx) = \frac{1}{n} \sum_{i=1}^N \xi_V(\bx_i)$, with $\xi_V(\bx) = \bx^\top (\bI-\bQ) \bx$. Here $P_N = \frac{1}{N} \sum_{i=1}^N \delta_{{\bx}_i}$ denotes the empirical measure, based on $\mathcal{X}$, i.e. $P_N(A) = \frac{1}{n} \sum_{i=1}^n \one\{{\bx}_i \in A\}$, for any set $A$. 

 MoM does not directly minimimze the empirical risk \eqref{e1}. Instead, it starts with a partitioning of the data into $L$ groups $B_1,\dots,B_L$ (i.e. $B_\ell \in 2^{[n]}$, $\cup_{\ell=1}^L B_\ell = [n]$ and $B_\ell \cap B_{\ell^\prime} = \emptyset$ for $\ell \neq \ell^\prime$), 
 each of which contains exactly $B$ many elements (repudiating a few observations when $L$ does not divide $n$). One can make a uniform at random assignment of the partitions or shuffle them \textit{on the run} \cite{lecue2020robust1}. 
 Let $P_{B_\ell}$ denote the empirical distribution of $\{{\bx}_i\}_{i \in B_\ell}$. For notational simplicity, we write, $\mu f = \int f d\mu$. We formulate the \textbf{M}edian \textbf{o}f \textbf{M}eans \textbf{P}rincipal \textbf{C}omponents \textbf{A}nalysis (MoMPCA) as the solution to the minimization problem of the following objective function,
 \begin{align}
     \text{MoM}_L^N(\xi_V) = \text{Median}\left(\left\{P_{B_\ell}\xi_V: \ell =1,\dots, L\right\}\right)        = \frac{1}{B} \text{Median}\left(\sum_{i\in B_1} \xi_V({\bx}_i) ,\dots, \sum_{i\in B_L} \xi_V({\bx}_i)\right) \label{e2}
 \end{align}
subject to $ V  \in \mathcal{V}_d$. A pictorial representation of objective \ref{e2} is given in Fig. \ref{fig:my_label}. Compared to their ERM counterparts, MoM estimators appear to exhibit stronger robustness as under mild assumptions, the outliers can affect only a subset of the partitions while the others remain free of outliers. Taking the median over partitions negates the effect of such contrived partitions, thus neutralizing the influence of outliers. Works like \cite{lecue2019learning,rodriguez2019breakdown} formally analyzes the robustness of the MoM estimators through breakdown points. 
\subsection{Optimization}
Optimizing \eqref{e2} is tractable via gradient-based methods. It is easy to see that we can write $\bQ=\bV \bV^\top$, where $\bV$ is a $p \times d$ real column-orthonormal matrix. Thus, making a variable transform, we can write, \eqref{e2} as 
\begin{equation}\label{e3}
\text{MoM}_L^N(g_V) = \frac{1}{B} \text{Median}\left(\sum_{i\in B_1} g_V({\bx}_i) ,\dots, \sum_{i\in B_L} g_V({\bx}_i)\right) 
\end{equation}
 where $g_V(\bx)= \bx^\top (I-\bV\bV^\top) \bx$. At the $t$-th iteration, let the median partition be $\ell_{\text{med}}^{(t)}$, i.e. $\ell_t\in [L]$ be such that $\text{MoM}_L^n(g_{\bV^{(t)}}) = \frac{1}{B}\sum_{i \in B_{\ell_t}} g_{\bV^{(t)}}({\bx}_i)$. According to \cite{lecue2020robust}, $\nabla \text{MoM}_L^N(g_{\bV^{(t)}}) = \frac{2}{B}\sum_{i \in B_{\ell_{\text{med}}^{(t)}}} {\bx}_i {\bx}_i^\top \bV^{(t)}$. However, after taking a gradient step, $\bV^{(t)} - \eta \frac{1}{B}\sum_{i \in B_{\ell_{\text{med}}^{(t)}}} {\bx}_i {\bx}_i^\top \bV^{(t)}$, the resultant matrix may not be column orthonormal. Following \cite{arora2012stochastic}, we apply a Gram-Schimtz algorithm to make the columns of this resultant matrix orthonormal. Let $\mathcal{P}_{\text{orth}}(\bV)$ denote the resultant matrix after orthonormalizing the columns of $\bV$. One can thus use the following gradient descent update:
 \[\bV^{(t+1)} \leftarrow \mathcal{P}_{\text{orth}} \bigg(\bV^{(t)} - \eta \frac{1}{B}\sum_{i \in B_{\ell_{\text{med}}^{(t)}}} {\bx}_i {\bx}_i^\top \bV^{(t)}\bigg).\]
 Algorithm \ref{algo} gives a formal description of the proposed gradient descent updates to optimize \eqref{e2}.
 \begin{figure*}
     \centering
    \begin{tikzpicture}
 

\node [draw,
    fill=SeaGreen, 
    minimum width=2cm, 
    minimum height=1.2cm, 
]  (data) at (0,0) {$\bX_1, \dots, \bX_N$};

\node[draw,
    minimum size=0.6cm,
    fill=Rhodamine!50,
   above right= 2cm and -1 cm of data
] (sum) {Random Partitions};

 
 
 
\node [draw,
    fill=Goldenrod,
    minimum width=2cm,
    minimum height=1.2cm,
    right=2cm of data,
    text width=2cm,
    align=center
]  (partition 2) {Partition 2: $\{\bX_i\}_{i \in B_2}$};

\node [draw,
    fill=Goldenrod,
    minimum width=2cm,
    minimum height=1.2cm,
    above =0.5 cm of partition 2,
    text width=2cm,
    align=center
]  (partition 1) {Partition 1: \\ $\{\bX_i\}_{i \in B_1}$};

\node [draw,
    fill=Goldenrod,
    minimum width=2cm,
    minimum height=1.2cm,
    below =1 cm of partition 2,
    text width=2cm,
    align=center
]  (partition L) {Partition L: $\{\bX_i\}_{i \in B_L}$};
 
 \node [draw,
    fill=Goldenrod,
    minimum width=2cm,
    minimum height=1.2cm,
    right =1cm of partition 1,
]  (f1) {$\sum\limits_{i \in B_1} \bX_i^\top (I - \bV \bV^\top) \bX_i$};

 \node [draw,
    fill=Goldenrod,
    minimum width=2cm,
    minimum height=1.2cm,
    right =1cm of partition 2,
]  (f2) {$\sum\limits_{i \in B_2} \bX_i^\top (I - \bV \bV^\top) \bX_i$};

 \node [draw,
    fill=Goldenrod,
    minimum width=2cm,
    minimum height=1.2cm,
    right =1cm of partition L,
]  (fL) {$\sum\limits_{i \in B_L} \bX_i^\top (I - \bV \bV^\top) \bX_i$};

\node [draw,
    fill=SpringGreen, 
    minimum width=2cm, 
    minimum height=1.2cm,
    right=1.5cm of  f2
] (mom) {$\text{MoM}_{L}^N(\xi_{\bV})$};
 \node[draw,
    minimum size=0.6cm,
    fill=Rhodamine!50,
   above right= 2cm and -1 cm of partition 2,
   text width=4cm,
   align=center
] (sum1) {Compute Empirical Risks \\ on Each Partition};
 \node[draw,
    minimum size=0.6cm,
    fill=Rhodamine!50,
   above right= 2cm and -1 cm of f2
] (sum2) {Compute Median};
 
\path (partition 2) -- (partition L) node [black, font=\huge, midway, sloped] {$\dots$};
\path (f2) -- (fL) node [black, font=\huge, midway, sloped] {$\dots$};
\draw[-stealth] (data.east) -- (partition 2.west)
    node[midway,above]{};
    
    \draw[-stealth] (data.east) -- (partition 1.west)
    node[midway,above]{};
    \draw[-stealth] (data.east) -- (partition L.west)
    node[midway,above]{};
     \draw[-stealth] (partition 1.east) -- (f1.west)
    node[midway,above]{};
     \draw[-stealth] (partition 2.east) -- (f2.west)
    node[midway,above]{};
    \draw[-stealth] (partition L.east) -- (fL.west)
    node[midway,above]{};
     \draw[-stealth] (f1.east) -- (mom.west)
    node[midway,above]{};
     \draw[-stealth] (f2.east) -- (mom.west)
    node[midway,above]{};
     \draw[-stealth] (fL.east) -- (mom.west)
    node[midway,above]{};
 
 
 
 

\end{tikzpicture}
     \caption{A pictorial illustration of median-of-means PCA objective.}
     \label{fig:my_label}
 \end{figure*}

\begin{algorithm}[t]
\caption{Median of Means Principal Component Analysis (MoMPCA)}\label{algo}
\begin{algorithmic}
\State \textbf{Input}: The dataset $\mathcal{X}=\{{\bx}_1,\dots,{\bx}_n\}$, $L$, $d$, $\eta$. 
\State \textbf{Output}:  $\widehat{\bQ}_{N,L}$.
\State Initialization: Randomly partition $[N]$ into $L$ many partitions of equal length. 
\State Compute the feature-wise median $\bmu$ of $\mathcal{X}$. Replace ${\bx}_i$ by ${\bx}_i - \bmu$.
\State Compute the first $d$ eigenvectors of $\sum_{i=1}^n {\bx}_i {\bx}_i^\top$ and stack them as columns of $V^{(0)}$.
\Repeat
\State Compute $v_\ell \leftarrow \sum_{i \in B_\ell} {\bx}_i^\top (I- \bV^{(t)}(\bV^{(t)})^\top){\bx}_i$ 
\State Compute the median class $\ell_{\text{med}}$, such that $v_{\ell_{\text{med}}} = \text{median}(\{v_\ell\}_{\ell=1}^L)$
\[\bV^{(t+1)} \leftarrow \mathcal{P}_{\text{orth}} \bigg(\bV^{(t)} -  \frac{\eta}{B}\sum_{i \in B_{\ell_{\text{med}}^{(t)}}} {\bx}_i {\bx}_i^\top \bV^{(t)}\bigg).\]
 \Until{objective $\bV^{(t)}$ converges}
 \State Compute $\widehat{\bQ}_{N,L} \leftarrow \bV^{(t)}(\bV^{(t)})^\top $.
\end{algorithmic}
\end{algorithm}

 \subsection{Centering} One should note that the aforementioned formulation is for centered observations. If the data is not centered, one must first center the observations by subtracting a robust estimate of location from each point. This is to emphasize that under outlier contamination, the sample mean may not be a reasonable estimate for the measure of central tendency for the data, and one has to resort to other robust estimates of the location, such as component-wise median or data depth based measures, to center the data.

\subsection{Time Complexity}
It is a well known fact that in a dataset with $n$ number of datapoints and $p$ many features, the time complexity of vanilla PCA is $O(\min \{n^3,p^3\})$. For our method, we note that, the computation of $v_\ell$'s require $\mathcal{O}(k p^2 (p \vee b))$ time. This is because, $v_\ell = \text{trace}(\bV^{(t)} (\bV^{(t)})^\top \sum_{i \in B_\ell} {\bx}_i {\bx}_i^\top)$ and requires $\mathcal{O}(p^2 b + p^2 d + p^3 + p) = \mathcal{O}(p^2 (b\vee p) )$, taking into account all the matrix multiplications and trace computation. The median computation takes $\mathcal{O}(L)$ time. For the computation of $\bV^{(t+1)}$, we note that,  the gradient computation takes $\mathcal{O}(b p^2 + p^2 d + p^2 ) = \mathcal{O}(p^2 b) $ time and the Gram-Schmidt orthogonalization takes $\mathcal{O}(pd^2)$ time, making a total of $\mathcal{O}(p^2b + pd^2)$ time. Thus the per iteration complexity of the loop is $\mathcal{O}(p^2 (b\vee p) + p^2b + pd^2 + L) $ time. In general, $b >p$ and the per iteration complexity becomes $\mathcal{O}(p^2 b + p d^2 + L)$.

\section{Theoretical Properties}
\label{theory}

\subsection{Notation and Setup}
Instead of performing the analysis on a real vector space as was done in our recent MoM-based robust clustering work \cite{paul2021uniform}, here we develop our proofs for a (real) Hilbert space in context to robust PCA. Suppose the data is observed in a real separable Hilbert space $\mathcal{H}$, i.e., $\{X_i\}_{i \in [n]} \subset \mathcal{H}$. Let the norm in $\mathcal{H}$ be denoted by $\|\cdot\|$. A linear operator $L : \Hilbert \to \Hilbert$, is said to be Hilbert-Schmidt if $\sum_{i \ge 1} \|L e_i\|^2 < \infty$ and the sum is independent of the chosen orthonormal basis $\{e_i\}_{i \ge 1}$ (of $\Hilbert$). Let $HS(\Hilbert)$ be the set of all Hilbert-Schmidt linear operators. An inner product on $HS(\Hilbert)$ is defined as, $\langle L_1, L_2 \rangle = \sum_{i \ge 1} \langle L_1 e_i, L_2 e_i \rangle$. It can be shown that $HS(\Hilbert)$ is also a separable (real) Hilbert space. The trace of a linear operator is defined as $\text{tr} (L) = \sum_{i \ge 1} \langle L e_i , e_i \rangle$.

For any $x, y \in \Hilbert \setminus \{0\}$, the outer product operator $x \otimes y$ is defined as $(x \otimes y)(z) = \langle y, z \rangle x $. This outer product satisfies the following properties.
\begin{itemize}
    \item $\|x \otimes y \|_{HS(\Hilbert)} = \|x\| \|y\|$.
    \item $\text{tr}(x \otimes y) = \langle x , y \rangle$.
    \item $\langle L , x \otimes y \rangle_{HS(\Hilbert)} = \langle L y , x \rangle$
\end{itemize}
Recall that if $U$ is an orthogonal projector operator in $\Hilbert$, then  $U^2 = U$. Moreover, $\|U x\|^2  = \langle x, U x \rangle \le \|x\|^2$ and $\langle x \otimes y , U \rangle_{HS(\Hilbert)} = \langle U x , U y \rangle $. $U$ has rank $d< \infty$, iff it is Hilbert-Schmidt and 
\[\|U\|_{HS(\Hilbert)}^2 = d ; \quad \text{tr}(U) = d.\]

Suppose $V$ is a closed subspace of $\Hilbert$, then ${\bPi}_{V}$ denotes the (unique) orthogonal projector onto $V$. $V^\perp$ represents the orthogonal complement of $V$.

We aim to find the ``best" subspace of dimension $d$ in this Hilbert space that minimizes the reconstruction error. Hence we consider the following objective, similar to equation \ref{e2}

We assume that the dataset can be split into two categories: the set of inliers ($\I$) and the set of outliers ($\cO$), i.e., $[N]=\I \cup \cO$. We will assume the following in the data generation process.

\begin{ass}\label{a1}
$\{{\bx}_i\}_{i \in \I}$ are independently and identically distributed (i.i.d) according to the distribution $P$.
\end{ass}
\begin{ass}
\label{a2}
$\mathbb{E}_{{\bx} \sim P}({\bx}) = \mathbf{0}$ and $\mu_4 =  \E_{X \sim P}\|{\bx}\|^4  < \infty$.
\end{ass}

\begin{ass}
\label{a3}
$\exists \, \eta >0$, such that $L>(2+\eta)|\cO|$.
\end{ass}

We emphasize that we do not make any assumptions about the distributions of the outlying observations. They are even allowed to be dependent and may even come from some heavy-tailed distributions. Assumptions A\ref{a1} and A\ref{a2} state that the inlying observations are independently and identically generated from some distribution with zero mean and a finite fourth moment. To see the significance of assumption A\ref{a3}, one should note that by the pigeonhole principle, (strictly) less than half of the partitions may contain an outlying observation. More than half of the partitions, thus, do not contain an outlier. Since the median is driven by the majority half of the partitions, naturally by A\ref{a3}, since the median is only affected by the objective function value of the majority half of these partitions, the MoM estimates can be expected only to be based on the inlying observations. We note that our analysis does not require ``$L> 4|\cO|$" as in \cite{lecue2020robustcl} but only requires that $L>(2+\eta)|\cO|$, which is a much weaker condition. Moreover, opposed to \cite{lecue2020robustcl}, A\ref{a3} has nice practical interpretations.

Let 
\begin{align*}
    \xi_V(\bx)  = \|\bx - {\bPi}_V \bx\|^2 & = \|{\bPi}_{V^\perp} \bx\|^2 = \langle {\bPi}_{V^\perp} \bx, {\bPi}_{V^\perp} \bx \rangle  = \langle {\bPi}_{V^\perp},  \bx \otimes \bx \rangle_{HS(\Hilbert)}.
\end{align*}
For notational simplicity, for any function $g$, We define the operator, $\text{MoM}_L^N (\cdot)$ as follows:
\[\text{MoM}_L^N (g) = \frac{1}{B} \text{Median}\left(\sum_{i\in B_1} g({\bx}_i) ,\dots, \sum_{i\in B_L} g({\bx}_i)\right).\]

 It is easy to observe that the MoMPCA problem in section \ref{formulation} can be restated as, 
 \[\min_{V \in  \mathcal{V}_d} \text{MoM}_L^N (\xi_V).\]
Here, $\mathcal{V}_d$ denotes the set of all $d$-dimensional subspaces of $\Hilbert$. Let the unique minimizer to the above problem be $\hat{V}_d = \argmin_{V \in  \mathcal{V}_d} \text{MoM}_L^N (\xi_V)$. Let $V^\ast_d = \argmin_{V \in  \mathcal{V}_d} P \xi_V$ be the global population minimizer. The excess risk of the estimate, $\hat{V}_d$ is given by, 
\begin{align}\label{test_error}
    \mathfrak{R}(\hat{V}_d) = P \xi_{\hat{V}_d} - P \xi_{V^\ast_d} = \int \xi_{\hat{V}_d} dP - \inf_{V \in  \mathcal{V}_d}\int \xi_{V} dP
\end{align}

We note that  $P \xi_{V^\ast_d} \le P \xi_{V}$ and $\text{MoM}_L^N (\xi_{\hat{V}_d}) \le \text{MoM}_L^N (\xi_{V})$, for all $V  \in \mathcal{V}_d$. Thus, 
\begin{align}\small
     |P \xi_{\hat{V}_d} - P \xi_{V^\ast_d} | 
    = & P\xi_{\hat{V}_d}-P \xi_{V^\ast_d} \nonumber \\
    = & (P\xi_{\hat{V}_d} - \text{MoM}_L^N (\xi_{\hat{V}_d})) + (\text{MoM}_L^N (\xi_{\hat{V}_d}) \nonumber \\
    & - \text{MoM}_L^N (\xi_{V^\ast_d}) + \left( \text{MoM}_L^N ( \xi_{V^\ast_d}) - P  \xi_{V^\ast_d}\right) \nonumber \\
     \le & (P\xi_{\hat{V}_d} - \text{MoM}_L^N (\xi_{\hat{V}_d})) + \left( \text{MoM}_L^N (\xi_{V^\ast_d}) - P \xi_{V^\ast_d}\right) \nonumber\\
    \le & 2 \sup_{V \in  \mathcal{V}_d} |\text{MoM}^N_L (\xi_{V}) - P \xi_V |. \label{sup bd}
\end{align}
It is thus, enough to prove bounds on the uniform deviation $\sup_{V \in  \mathcal{V}_d} |\text{MoM}^N_L (\xi_{V}) - P \xi_V |$. Towards that, we will first derive bounds on the Rademacher complexity in the following section, followed by our main results.

\subsection{Bounds on the Rademacher Complexity}
For our theoretical understanding of MoMPCA, we need to compute the Rademacher complexity of the function class $\Xi_d = \{\xi_V: V  \in \mathcal{V}_d\}$. We first recall the definition of Rademacher complexity. Let ${\by}_1,\dots,{\by}_m \overset{i.i.d.}{\sim}P$ be random variables in some space $\mathcal{S}$. Suppose $\mathcal{F} \subseteq {\Real}^\mathcal{S}$ be a class of functions from $\mathcal{S}$ to $\Real$ and let  $\sigma_1,\dots, \sigma_m$ be i.i.d. Rademacher random variables. The empirical Rademacher complexity, based on $\mathcal{Y}=\{{\by}_1,\dots,{\by}_m\}$ is defined as:
\[\hat{R}_{\mathcal{Y}}(\mathcal{F}) = \frac{1}{m} \E \left( \sup_{ f \in \mathcal{F} }\sum_{i=1}^m \sigma_i f({\by}_i)\bigg|\mathcal{Y}\right).\]
Similarly, the population Rademacher complexity is defined as 
\[R_{m}(\mathcal{F}) = \E(R_{\mathcal{Y}}(\mathcal{F})) = \frac{1}{m} \E\sup_{ f \in \mathcal{F} }\sum_{i=1}^m \sigma_i f({\by}_i).\]
We now compute the Rademacher complexity of the function class $\{\xi_V: V  \in \mathcal{V}_d\}$ in Theorem \ref{th1}.


\begin{theorem}\label{th1}
Suppose $\Xi_d = \{\xi_V: V   \in \mathcal{V}_d\}$. Let ${\by}_1,\dots,{\by}_m \overset{\text{i.i.d}}{\sim} P$, with $P$ satisfying Assumption \ref{a2}. Then, \(R_m(\Xi_d)  \le \sqrt{\frac{d\mu_4}{m}}.\)
\end{theorem}
\begin{proof}
Suppose ${\by}_1,\dots,{\by}_m \overset{\text{i.i.d}}{\sim} P$ and let $\mathcal{Y}=\{{\by}_1,\dots,{\by}_m\}$. The empirical Rademacher complexity is thus, given by, 
\begingroup
\allowdisplaybreaks
\begin{align}
    \widehat{R}_{\Xi_d}(Q) & = \frac{1}{m} \E_{\bsigma} \sup_{ V  \in \mathcal{V}_d}\sum_{i=1}^m \sigma_i \xi_V({\by}_i) \nonumber\\
    & = \frac{1}{m} \E_{\bsigma} \sup_{ V  \in \mathcal{V}_d}\sum_{i=1}^m \sigma_i ( \|{\by}_i\|_2^2 - \langle {\bPi}_{V},  {\by}_i \otimes {\by}_i \rangle_{HS(\Hilbert)}) \nonumber \\
   & = \frac{1}{m} \E_{\bsigma} \sup_{ V  \in \mathcal{V}_d} \left\langle {\bPi}_{V},  \sum_{i=1}^m \sigma_i {\by}_i \otimes {\by}_i \right\rangle_{HS(\Hilbert)} \nonumber \\
   & \le  \frac{1}{m} \E_{\bsigma} \sup_{ V  \in \mathcal{V}_d}  \|{\bPi}_{V}\|_{HS(\Hilbert)} \left\|\sum_{i=1}^m \sigma_i {\by}_i \otimes {\by}_i\right\| _{HS(\Hilbert)} \label{rad1} \\
   & \le  \frac{\sqrt{d}}{m} \E_{\bsigma}  \left\|\sum_{i=1}^m \sigma_i {\by}_i \otimes {\by}_i\right\| _{HS(\Hilbert)} \nonumber\\
    & \le  \frac{\sqrt{d}}{m} \sqrt{\E_{\bsigma}  \left\|\sum_{i=1}^m \sigma_i {\by}_i \otimes {\by}_i\right\| _{HS(\Hilbert)}^2 }\label{rad2}\\
    & =  \frac{\sqrt{d}}{m} \sqrt{  \sum_{i=1}^m \left\|  {\by}_i \otimes {\by}_i\right\| _{HS(\Hilbert)}^2 }\nonumber\\
     & =  \frac{\sqrt{d}}{m} \sqrt{  \sum_{i=1}^m \|{\by}_i\|^4 }\nonumber
\end{align}
\endgroup

Here, inequalities \eqref{rad1} and \eqref{rad2} follow from appealing to Cauchy-Schwartz and Jensens's inequalities, respectively. Thus, 
\begingroup
\allowdisplaybreaks
\begin{align}
   R_m(\Xi_d) & = \E_{\mathcal{Y}} \widehat{R}_{\mathcal{X}}(\Xi_d) \nonumber \\
   &= \frac{1}{m} \sqrt{d} \, \E \left[\sqrt{    \sum_{i=1}^m \|{\by}_i\|^4} \right] \nonumber \\
  & \le \frac{1}{m} \sqrt{d}  \sqrt{    \sum_{i=1}^m\E \|{\by}_i\|^4} \label{rad3}\\
  & =  \sqrt{\frac{d\mu_4}{m}}.\nonumber
\end{align}
\endgroup
Inequality \eqref{rad3} follows from Jensen's inequality.
\end{proof}

\subsection{Uniform Concentration Bounds}
We are now ready to state and prove the uniform concentration bound result in Theorem \ref{th2}.  Theorem \ref{th2} asserts that with a high probability, $\sup_{ V  \in \mathcal{V}_d} |\text{MoM}_L^N (\xi_V) - P \xi_V | \lesssim  \sqrt{\frac{L}{N}} +  \frac{\sqrt{|\I|}}{N}$ with a very high probability.
 \begin{theorem}\label{th2}
Under A\ref{a1}--\ref{a3}, and $N>L$, with probability at least, $1-2 e^{-2L \left(\frac{2}{4+\eta} - \frac{|\cO|}{L}\right)^2}$,
\[\sup_{ V  \in \mathcal{V}_d} |\text{MoM}_L^N (\xi_V) - P \xi_V | \le C \sqrt{\E \|{\bx}\|^4}\left( \sqrt{\frac{L}{N}} +  \frac{\sqrt{d|\I|}}{N}\right),\]
with $C= 2\max\left\{\sqrt{\frac{8(4+\eta)}{\eta}}, \frac{16 (4+\eta)}{\eta}\right\}$.
\end{theorem}
\begin{proof} Suppose $\epsilon>0$. We will first bound the probability of $\sup_{ V  \in \mathcal{V}_d} |\text{MoM}_L^N (\xi_V) - P \xi_V |> \epsilon$. To do so, we will individually bound the probabilities of the events, $\sup_{ V  \in \mathcal{V}_d}( \text{MoM}_L^N (\xi_V) - P \xi_V) >\epsilon$ and $\sup_{ V  \in \mathcal{V}_d}  ( P \xi_V - \text{MoM}_L^N (\xi_V)) > \epsilon$. We note that if $\sup_{ V  \in \mathcal{V}_d}\sum_{\ell = 1}^L \one\left\{(P-P_{B_\ell})\xi_V > \epsilon\right\} > \frac{L}{2}$, then, $\sup_{ V  \in \mathcal{V}_d}  ( P \xi_V - \text{MoM}_L^N (\xi_V)) > \epsilon$. Here $\one\{\cdot\}$ denote the indicator function. Let $\varphi(t) = (t-1) \one\{1 \le t \le 2\} + \one\{t >2\}$. Clearly,
\begin{equation}
    \label{eq6}
    \one\{t \ge 2\} \le \varphi(t) \le \one\{t \ge 1\}.
\end{equation}We observe that, 
\begingroup
\allowdisplaybreaks
\begin{align}
   & \sup_{ V  \in \mathcal{V}_d}\sum_{\ell = 1}^L \one\left\{(P-P_{B_\ell})\xi_V > \epsilon\right\}\nonumber\\
    \le & \sup_{ V  \in \mathcal{V}_d}\sum_{\ell \in \cL} \one\left\{(P-P_{B_\ell})\xi_V > \epsilon\right\} + |\cO|\nonumber\\
    \le & \sup_{ V  \in \mathcal{V}_d}\sum_{\ell \in \cL}  \varphi\left(\frac{2(P-P_{B_\ell})\xi_V}{\epsilon}\right) + |\cO|\nonumber\\
    \le & \sup_{ V  \in \mathcal{V}_d}\sum_{\ell \in \cL}  \E \varphi\left(\frac{2(P-P_{B_\ell})\xi_V}{\epsilon}\right)  + |\cO| +\sup_{ V  \in \mathcal{V}_d}\sum_{\ell \in \cL}  \bigg[ \varphi\left(\frac{2(P-P_{B_\ell})\xi_V}{\epsilon}\right)   - \E \varphi\left(\frac{2(P-P_{B_\ell})\xi_V}{\epsilon}\right)\bigg]. \label{eq1}
\end{align}
\endgroup
To bound $\sup_{ V  \in \mathcal{V}_d}\sum_{\ell = 1}^L \one\left\{(P-P_{B_\ell})\xi_V > \epsilon\right\}$, we will first bound the quantity, $\E \varphi\left(\frac{2(P-P_{B_\ell})\xi_V}{\epsilon}\right)$. We observe that, 
\begingroup
\allowdisplaybreaks
\begin{align}
\small 
  \E \varphi\left(\frac{2(P-P_{B_\ell})\xi_V}{\epsilon}\right) 
  \le & \E \left[ \one\left\{\frac{2(P-P_{B_\ell})\xi_V}{\epsilon} > 1 \right\}\right] \nonumber \\
  = & \sP \left[ (P-P_{B_\ell})\xi_V>\frac{\epsilon}{2} \right] \nonumber \\ 
  \le & \frac{4}{\epsilon^2} \text{Var}\left((P-P_{B_\ell})\xi_V\right) \label{eq3} \\
  = & \frac{4}{\epsilon^2} \text{Var}\left(P_{B_\ell}\xi_V\right) \nonumber \\
  = & \frac{4}{B\epsilon^2} \text{Var}\left(\|{\bPi}_{V^\perp} {\bx}\|^2\right) \nonumber \\
\le & \frac{4}{B\epsilon^2} \E \|{\bPi}_{V^\perp} {\bx}\|^4 \nonumber \\
\le & \frac{4}{B\epsilon^2} \E \| {\bx}\|^4 \nonumber
\end{align} 
\endgroup
Here equation \eqref{eq3} follows from Chebyshev's inequality. 
We now concentrate on bounding the term $\sup_{ V  \in \mathcal{V}_d}\sum_{\ell \in \cL}  \bigg[ \varphi\left(\frac{2(P-P_{B_\ell})\xi_V}{\epsilon}\right)   - \E \varphi\left(\frac{2(P-P_{B_\ell})\xi_V}{\epsilon}\right)\bigg]$. Appealing to Theorem 26.5 of \cite{shalev2014understanding} we observe that, with probability at least $1-e^{-2L \delta^2}$, $\forall  V  \in \mathcal{V}_d$,
\begin{align}
  \frac{1}{L}\sum_{\ell \in \cL}   \varphi\left(\frac{2(P-P_{B_\ell})\xi_V}{\epsilon}\right) 
  \le & \E\left[\frac{1}{L}\sum_{\ell \in \cL}  \varphi\left(\frac{2(P-P_{B_\ell})\xi_V}{\epsilon}\right) \right] + 2\E\left[\sup_{ V  \in \mathcal{V}_d}\frac{1}{L}\sum_{\ell \in \cL}\sigma_\ell  \varphi\left(\frac{2(P-P_{B_\ell})\xi_V}{\epsilon}\right) \right] + \delta.  \label{eq5}
\end{align}
Here $\{\sigma_\ell\}_{\ell \in \mathcal{L}}$ are independent Rademacher variables. Suppose that $\{\xi_i\}_{i=1}^n$ are independent Rademacher random variables and independent of  $\{\sigma_\ell\}_{\ell \in \mathcal{L}}$. From equation \eqref{eq5}, we get,  
\begingroup
\allowdisplaybreaks
\begin{align}
     & \frac{1}{L}\sup_{ V  \in \mathcal{V}_d}\sum_{\ell \in \cL}  \bigg[ \varphi\left(\frac{2(P-P_{B_\ell})\xi_V}{\epsilon}\right)  - \E \varphi\left(\frac{2(P-P_{B_\ell})\xi_V}{\epsilon}\right)\bigg] \nonumber\\
     \le & 2\E\left[\sup_{ V  \in \mathcal{V}_d}\frac{1}{L}\sum_{\ell \in \cL}\sigma_\ell  \varphi\left(\frac{2(P-P_{B_\ell})\xi_V}{\epsilon}\right) \right] + \delta \nonumber \\
     \le & \frac{4}{L \epsilon}\E\left[\sup_{ V  \in \mathcal{V}_d}\sum_{\ell \in \cL}\sigma_\ell  (P-P_{B_\ell})\xi_V \right]+ \delta. \label{eq7} 
     \end{align}
     Equation \eqref{eq7} follows from the fact that $\varphi(\cdot)$ is 1-Lipschitz and appealing to Lemma 26.9 of \cite{shalev2014understanding}. We now introduce a phantom sample $\mathcal{X}^\prime=\{{\bx}_i^\prime, \dots, {\bx}_n^\prime\}$, which are i.i.d. and follows the law $P$, independent of $\mathcal{X}$. Thus, equation \eqref{eq7} further equals the following quantity.
     \begin{align}
     = & \frac{4}{L \epsilon}\E\left[\sup_{ V  \in \mathcal{V}_d}\sum_{\ell \in \cL}\sigma_\ell  \E_{\mathcal{X}^\prime}\left((P^\prime_{B_\ell}-P_{B_\ell})\xi_V\right) \right]+ \delta \nonumber \\
     \le & \frac{4}{L \epsilon}\E\left[\sup_{ V  \in \mathcal{V}_d}\sum_{\ell \in \cL}\sigma_\ell  (P^\prime_{B_\ell}-P_{B_\ell})\xi_V \right]+ \delta \nonumber \\
     = & \frac{4}{L \epsilon}\E\left[\sup_{ V  \in \mathcal{V}_d}\sum_{\ell \in \cL}\sigma_\ell  \frac{1}{B}\sum_{i \in B_\ell}(\xi_V({\bx}_i^\prime)-\xi_V({\bx}_i)) \right]+ \delta \nonumber \\
     = & \frac{4}{B L \epsilon}\E\left[\sup_{ V  \in \mathcal{V}_d}\sum_{\ell \in \cL}\sigma_\ell  \sum_{i \in B_\ell} \xi_i(\xi_V({\bx}_i^\prime)-\xi_V({\bx}_i)) \right]+ \delta \label{eq8} \\
     = & \frac{4}{N \epsilon}\E\left[\sup_{ V  \in \mathcal{V}_d}\sum_{\ell \in \cL}  \sum_{i \in B_\ell} \sigma_\ell \xi_i(\xi_V({\bx}_i^\prime)-\xi_V({\bx}_i)) \right]+ \delta \nonumber \\
     \le & \frac{4}{N \epsilon}\E\left[\sup_{ V  \in \mathcal{V}_d}\sum_{\ell \in \cL}  \sum_{i \in B_\ell} \sigma_\ell \xi_i(\xi_V({\bx}_i^\prime)+\xi_V({\bx}_i)) \right]+ \delta \nonumber \\
     = & \frac{4}{N \epsilon}\E\left[\sup_{ V  \in \mathcal{V}_d} \sum_{i \in \J} \gamma_i (\xi_V({\bx}_i^\prime)+\xi_V({\bx}_i)) \right] \label{eq9} \\
     = & \frac{8}{N \epsilon}\E\left[\sup_{ V  \in \mathcal{V}_d} \sum_{i \in \J} \gamma_i \xi_V({\bx}_i) \right]+ \delta \label{eq10} \\
     \le & \frac{8}{N \epsilon} \sqrt{d\mu_4 |\mathcal{J}|} + \delta \nonumber \\
     \le & \frac{8}{N \epsilon} \sqrt{d\mu_4 |\mathcal{I}|} + \delta. \label{eq11}
\end{align}
\endgroup
Equation \eqref{eq8} follows from observing that  $(\xi_V({\bx}_i^\prime)-\xi_V({\bx}_i)) \overset{d}{=} \xi_i(\xi_V({\bx}_i^\prime)-\xi_V({\bx}_i))$. In equation \eqref{eq9}, $\{\gamma_i\}_{i \in \mathcal{J}}$ are independent Rademacher random variables owing to their construction. Equation \eqref{eq10} can be derived from Theorem \ref{th1}.
Thus, combining equations \eqref{eq5}, \eqref{eq7}, and \eqref{eq11}, we conclude that,  with probability of at least $1-e^{-2L \delta^2}$, 
\begin{align}
    \sup_{ V  \in \mathcal{V}_d}\sum_{\ell = 1}^L \one\left\{(P-P_{B_\ell})\xi_V > \epsilon\right\} 
   \le & L \bigg(\frac{4}{B\epsilon^2} \mu_4 + \frac{|\cO|}{L} +  \frac{8}{N \epsilon} \sqrt{d \mu_4 |\mathcal{I}|} + \delta\bigg). \label{eq12}
\end{align}
We choose $\delta = \frac{2}{4+\eta} - \frac{|\cO|}{L}$ and \[\epsilon = 2\max\left\{\sqrt{\frac{8(4+\eta)\mu_4}{\eta}}\sqrt{\frac{L}{N}}, \frac{16\sqrt{d\mu_4}(4+\eta)}{\eta} \frac{\sqrt{|\I|}}{N}\right\}.\] The right hand side of \eqref{eq12} thus becomes strictly lesser than $\frac{L}{2}$.
Thus, we have shown that 
\begin{align*}
     \sP\left( \sup_{ V  \in \mathcal{V}_d} (P\xi_V - \text{MoM}^N_L (\xi_V))> \epsilon \right) \le e^{-2L \delta^2}.
 \end{align*}
Similarly, we can show that,
\begin{align*}
    \sP\left( \sup_{ V  \in \mathcal{V}_d} (\text{MoM}^N_L (\xi_V) -P\xi_V ) > \epsilon \right) \le e^{-2L \delta^2}.
\end{align*}
Merging the above two inequalities, we obtain: 
\[\sP\left( \sup_{ V  \in \mathcal{V}_d} |\text{MoM}^N_L (\xi_V) -P\xi_V | > \epsilon \right) \le 2e^{-2L \delta^2}.\]
Alternatively, with at least probability $1-2e^{-2L \delta^2}$,
\begingroup
\allowdisplaybreaks
\begin{align*}
     \sup_{ V  \in \mathcal{V}_d} |\text{MoM}^N_L (\xi_V) -P\xi_V | 
    \le & 2\max\bigg\{\sqrt{\frac{16(2+\eta)\mu_4}{\eta}}\sqrt{\frac{L}{N}},\frac{32\sqrt{d\mu_4}(2+\eta)}{\eta} \frac{\sqrt{|\I|}}{N}\bigg\}\\
    \le & C \sqrt{\mu_4} \left( \sqrt{\frac{L}{N}} +  \frac{\sqrt{d |\I|}}{N}\right).
\end{align*}
\endgroup
\end{proof}
From Theorem \ref{th2}, we can say that for any $ V  \in \mathcal{V}_d$, $\text{MoM}_L^N (\xi_V)$ and  $P \xi_V$ are close to each other with a high probability. Thus, one can expect their corresponding minimum values to be close enough with a high probability. Theorem \ref{th3} affirms this claim. We show that $ \mathfrak{R}(\hat{V}_d) \lesssim \sqrt{\E \|{\bx}\|^4}\left( \sqrt{\frac{L}{N}} +  \frac{\sqrt{d|\I|}}{N}\right)$ with a high probability.
\begin{theorem}
\label{th3}
Under A\ref{a1}--\ref{a3}, and $N>L$, with probability at least, $1-2 e^{-2L \left(\frac{2}{4+\eta} - \frac{|\cO|}{L}\right)^2}$,
\[ \mathfrak{R}(\hat{V}_d) \le 2 C \sqrt{\E \|{\bx}\|^4}\left( \sqrt{\frac{L}{N}} +  \frac{\sqrt{d|\I|}}{N}\right),\]
where $C$ is defined as in Theorem \ref{th2}.
\end{theorem}
\begin{proof} 
Using the bound derived in equation \eqref{sup bd}, we observe the following:
\begin{align}\small
   \mathfrak{R}(\hat{V}_d)
    \le & 2 \sup_{V \in  \mathcal{V}_d} |\text{MoM}^N_L (\xi_V) - P \xi_V | 
    \le  2C\sqrt{\E \|{\bx}\|^4}\left( \sqrt{\frac{L}{N}} +  \frac{\sqrt{d|\I|}}{N}\right). \label{eq13}
\end{align}
From Theorem \ref{th2}, equation \eqref{eq13} holds with probability at least $1-2 e^{-2L \left(\frac{2}{4+\eta} - \frac{|\cO|}{L}\right)^2}$, proving the desired result.
\end{proof}
\subsection{Inference for Real Vector Spaces}
In this section, we discuss the main implications of Theorem \ref{th3} for real vector spaces. We consider the special case of $\Hilbert = \Real^p$. We observe that $\frac{2}{4+\eta} - \frac{|\cO|}{L} \ge \frac{2}{4+\eta} - \frac{1}{2+\eta} =\frac{\eta}{(2+\eta)(4+\eta)}>0$. Thus, if $L \to \infty$, $1-2 e^{-2L \left(\frac{2}{4+\eta} - \frac{|\cO|}{L}\right)^2} \to 1$. We note that,

\[\mathfrak{R}(\hat{V}_d) \lesssim \sqrt{\E \|{\bx}\|^4}\left( \sqrt{\frac{L}{N}} +  \frac{\sqrt{d|\I|}}{N}\right).\]
If each component of the inliers were generated independently from the same distribution, it is easy to observe that, $\E \|{\bx}\|^4 \asymp p^2$. Thus, it is only natural to make the following assumption about $P$.
\begin{ass}
$\E_{{\bx} \sim P} \|{\bx}\|^4 = \Theta(p^2)$.
\end{ass}
Thus, with a high probability, 
\[\mathfrak{R}(\hat{V}_d) \lesssim p \left( \sqrt{\frac{L}{N}} +  \frac{\sqrt{d|\I|}}{N}\right) \le  p \left( \sqrt{\frac{L}{N}} +  \sqrt{\frac{d}{N}}\right). \]
In other words, if $L = o\left(\frac{N}{p^2}\right)$ and $ p^2 d = o(N)$, then $\mathfrak{R}(\hat{V}_d) \to 0$, with probability tending to $1$.

To formally state this result, we make the following two assumptions.
\begin{ass}
\label{a4}
$L \to \infty$, as $N \to \infty$.
\end{ass}
\begin{ass}
\label{a5}
$p^2 L, p^2 d = o(N)$.
\end{ass}
Such conditions apply naturally: as $n$ increases, so too must $L$ to preserve a proportion of the outlier-free partitions. Besides,  the increase of $L$ should be slower than $n$ so that each partition can be assigned with sufficient data points. Note that since $d \le p$, if $p^3 = o(N)$, then the second part of assumption \ref{a5} is satisfied. In other words, $N$ has to increase faster than the cube of the feature space dimensions. The condition implied by A\ref{a5}  i.e., $|\mathcal{O}| = o(n)$ is intuitively appealing, and widely accepted \cite{lecue2020robust1,pmlr-v130-staerman21a} as the number of outliers should be small by definition.

Before we state our result in Corollary \ref{cor1}, we recall that $X_n = O_P(a_n)$ if the sequence of random variables $\{X_n/a_n\}_{n \in \mathbb{N}}$ is tight \cite{athreya2006measure}. 
\begin{corollary}\label{cor1}
Then under,  A\ref{a1}--\ref{a5}, $\mathfrak{R}(\hat{V}_d) = O_P\left(p \left( \sqrt{\frac{L}{N}} +  \sqrt{\frac{d}{N}}\right)\right)$ 
and $\mathfrak{R}(\hat{V}_d) \xrightarrow{P} 0 $. 
\end{corollary}
\begin{proof}
We note that $\mathfrak{R}(\hat{V}_d)    \lesssim p \left( \sqrt{\frac{L}{N}} +  \sqrt{\frac{d}{N}} \right) $, with probability at least $1-2 e^{-2L \left(\frac{2}{4+\eta} - \frac{|\cO|}{L}\right)^2}$. By A\ref{a3} and \ref{a5}, 
$2 e^{-2L \left(\frac{2}{4+\eta} - \frac{|\cO|}{L}\right)^2}=o(1)$. Thus, $\sP\left(\mathfrak{R}(\hat{V}_d) = O\left(p \left( \sqrt{\frac{L}{N}} +  \sqrt{\frac{d}{N}} \right)\right) \right)\ge 1-o(1)$. Hence, $|P f_{\widehat{\bQ}_{N,L}} - P f_{\bQ^\ast}| = O_P\left(p \left( \sqrt{\frac{L}{N}} +  \sqrt{\frac{d}{N}} \right)\right)$

Under A\ref{a5}, 
$p \left( \sqrt{\frac{L}{N}} +  \sqrt{\frac{d}{N}} \right) = o(1) \, \implies \, \mathfrak{R}(\hat{V}_d) =o_P(1) $\footnote{$X_n = o_P(a_n)$ if $X_n/a_n \xrightarrow{P} 0$ \cite{athreya2006measure}.}. 
\end{proof}
To conclude that $\hat{V}_d \xrightarrow{P} V^\ast_d$, we need to ensure that $V_d^\ast$ is identifiable. Towards ensuring that, we make the following identifiablity assumption on $V_d^\ast$. This type of identifiability conditions are especially popular in clustering literature \cite{pollard1981,chakraborty2020entropy}. 
\begin{ass}
\label{a6}
For all $\epsilon>0$, there exists $\delta>0$, such that, if $\|V-V_d^\ast\|>\epsilon$, $\mathfrak{R}(V) \ge  \delta$.
\end{ass}
We will say that $V_n \to V$ if $\|V_n-V\|\to 0$. With this notion of convergence in the Frobenious sense, we are now ready to prove that $\hat{V}_d$ is consistent for $V_d^\ast$.
\begin{corollary}
Under A\ref{a1}-\ref{a6}, $\hat{V}_d \xrightarrow{P} V_d^\ast$.
\end{corollary}
\begin{proof}
We fix $\epsilon>0$. By A\ref{a6}, there exists $\delta>0$, such that, if $\|\hat{V}_d - V_d^\ast\|>\epsilon$, $\mathfrak{R}(\hat{V}_d) \ge \delta$. Thus, 
\begin{align*}
    \sP(\|\hat{V}_d - V_d^\ast\| > \epsilon) \le \sP\left( \mathfrak{R}(\hat{V}_d) \ge \delta \right)
    & \to 0,
\end{align*}
as $N \to \infty$, by appealing to Corollary \ref{cor1}. Thus, $\|\hat{V}_d - V_d^\ast\| \xrightarrow{P} 0 \iff \hat{V}_d \xrightarrow{P} V_d^\ast$. 
\end{proof}

\textbf{Remark: Cost of Robustness}
The above results under our paradigm, the  MoMPCA estimates admit an excess risk of $O\left(p \left( \sqrt{\frac{L}{N}} +  \sqrt{\frac{d}{N}}\right)\right) $. We observe that since $L\ge 1$, $p \left( \sqrt{\frac{L}{N}} +  \sqrt{\frac{d}{N}}\right) = \Omega\left(n^{-1/2}\right)$. Thus, our framework's convergence rate for MoMPCA is generally slower than its ERM counterpart, which has a rate of $O(n^{-1/2})$. This reiterates that there is ``no free lunch" when compromising robustness for the convergence rate, which is not unusual given that MoM operates on data contaminated with outliers. However, if the number of partitions $L$ increases slowly compared to $n$ (say, $L=O(\log n)$  and $|\cO|=O(\log n)$), the excess risk for MoMPCA estimates draws closer to its ERM counterpart at a rate of $\widetilde{O}(n^{-1/2})$.

\textbf{Remark: Choice of $L$} If the number of partitions can mitigate the effect of the outliers, under the proposed framework, the excess risk of the robust MoM estimates decreases with the block size at the rate of $b$ as $1/\sqrt{b} = \sqrt{L/n}$. Since $L$ can identifiability approximately as $2|\cO|$, the corresponding excess risk has the rate of  $O(\sqrt{|\cO|/n})$. However, the error bound $O(\sqrt{|\cO|/n})$ becomes vacuous if $|\cO| \propto n$. Thus, it is essential that $|\cO| = o(n)$ for our consistency guarantees to hold, as this enables us to select $L$ satisfying  A\ref{a4}-\ref{a5}.  One should note that one can achieve an error rate of $O(n^{(\beta-1)/2})$ if $|\cO| = O( n^{\beta}) $, for some $0<\beta<1$.

\textbf{Remark: Comparison with MCM-PCA} The recent array of works on geometric median-based approaches provide an attractive alternative to our proposal \cite{cardot2013efficient,cardot2017fast,fritz2012comparison,cohen2016geometric}. However, the two approaches are significantly different. The geometric median-based approaches focus on finding a robust estimate of the dispersion matrix using a geometric median-based loss function. These approaches then go on to find a suitable subspace for low-dimensional representation of the data by an eigen decomposition of this robust estimate of the covariance matrix. On the other hand, MoMPCA approaches changing the ERM problem and introducing a robust estimate of the projection error. Theoretically, we derive finite-sample error bounds in Hilbert space without imposing any assumption on the outliers, thus relaxing the i.i.d. assumption of the data distribution imposed in \cite{cardot2017fast}.

\textbf{Remark}
The MoMPCA method proposed in the paper can provide a more accurate analysis of data with outliers compared to existing methods. This is because the method is specifically designed to handle outlier data, and does not rely on unrealistic assumptions about the data distribution. The MoMPCA method does not make assumptions about the distribution of the outlier data. This allows for greater flexibility in dealing with different types of outlier behavior, making the method more useful in real-world applications where the nature of outliers may not be known in advance.
 Furthermore, the MoMPCA method provides a more interpretable understanding of PCA in the presence of outliers. This can be especially useful in applications where the insights provided by the analysis are important for decision-making. We note that MoMPCA does not require the data to have nice tail-conditions such as sub-Gaussian or sub-exponential  behavior as required by many relevant works in this direction \cite{paul2021uniform,lecue2020robust}, which are often unrealistic assumptions for real-world data. This means that the method can be applied to a wider range of datasets, making it more useful in practice.

\section{Experimental Results}
\label{experments}

\begin{table*}[ht]
    \centering
    \resizebox{\textwidth}{!}{
    \begin{tabular}{c c c c c c c c c}
    \hline
      $n$ & $p$ & rank(${\bx}_0$) & PCP \cite{candes2011robust} & PCPF \cite{pmlr-v48-chiang16} & MFRPCA \cite{chen2018robust} & RWL-AN \cite{NEURIPS2019_73f104c9} & MCM-PCA \cite{cardot2017fast}  & MoMPCA (Proposed)\\
      \hline
       $500$  &  $500$ &  $10$ & $1.6\times 10^{-3}$ & $\mathbf{1.5\times 10^{-3}}$ & $2.3\times 10^{-3}$ & $8.7\times 10^{-3}$ & $1.8\times 10^{-3}$ & $\mathbf{1.5\times 10^{-3}}$ \\
       $1000$ & $500$ & $ 10$ & $7.3\times 10^{-4}$ & $2.8\times 10^{-4}$ & $1.8\times 10^{-3}$ & $2.9\times 10^{-4}$ &  $2.7\times 10^{-4}$  & $\mathbf{2.1\times 10^{-4}}$\\
       $2000$ & $500$ & $10$ & $8.8\times 10^{-4}$ & $1.1\times 10^{-4}$ & $8.3\times 10^{-4}$ & $3.1\times 10^{-4}$ &  $6.4\times 10^{-5}$ & $\mathbf{5.6\times 10^{-5}}$\\
       $5000$ & $500$ & $10$ & $2.1\times 10^{-4}$ & $9.6\times 10^{-5}$ & $5.7\times 10^{-4}$ & $3.3\times 10^{-4}$ &  $8.1\times 10^{-6}$ & $\mathbf{7.2\times 10^{-6}}$\\
       $10000$ & $500$ & $10$ & $7.8\times 10^{-5}$ & $8.2\times 10^{-5}$ & $2.9\times 10^{-4}$ & $9.5\times 10^{-5}$ &  $4.3\times 10^{-7}$  & $\mathbf{3.8\times 10^{-7}}$\\
      \hline
    \end{tabular}
    }%
    \caption{Performance of different peer algorithms for low-rank matrix representation in terms of the relative reconstruction error.}
    \label{tab_syn}
\end{table*}

\begin{table*}[ht]
    \centering
    \resizebox{\textwidth}{!}{
    \begin{tabular}{c c c c c c c c c}
    \hline
     $n$ & $p$ & rank(${\bx}_0$) & PCP \cite{candes2011robust} & PCPF \cite{pmlr-v48-chiang16} & MFRPCA \cite{chen2018robust} & RWL-AN \cite{NEURIPS2019_73f104c9} & MCM-PCA \cite{cardot2017fast}  & MoMPCA (Proposed)\\
      \hline
       $500$  &  $500$ &  $10$ & $106.05$ & $104.13$ & $106.81$ & $8.14$ & $2.34$ &  $5.18$\\
       $1000$ & $500$ & $10$ & $320.43$ & $311.18$ & $321.32$ & $11.40$ &  $4.94$  & $10.31$\\
      $2000$ & $500$ & $10$ & $952.87$ & $932.65$ & $956.05$ & $26.76$ &  $10.35$ & $21.27$\\
       $5000$ & $500$ & $10$ & $3203.54$ & $3197.43$ & $3215.29$ & $49.03$ &  $21.02$  & $53.28$\\
       $10000$ & $500$ & $10$ & $9608.68$ & $9547.83$ & $9677.04$ & $103.42$ & $45.96$  & $98.04$\\
      \hline
    \end{tabular}
    }%
    \caption{Runtime Comparison of different peer algorithms for low-rank matrix representation.}
    \label{tab_runtime}
\end{table*}

\begin{figure*}
\centering
\subfloat[MoMPCA vs MCM-PCA]{\includegraphics[width=0.3\textwidth]{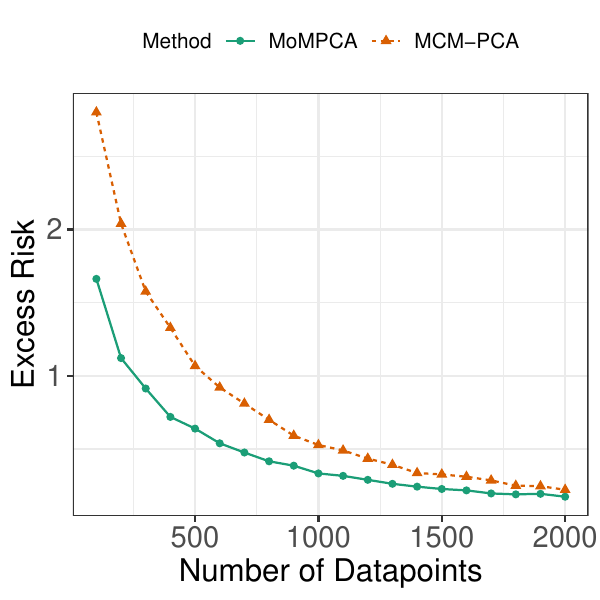}}\label{mommcm}
\subfloat[Vanilla PCA] {\includegraphics[width=0.3\textwidth,height=0.27\textwidth]{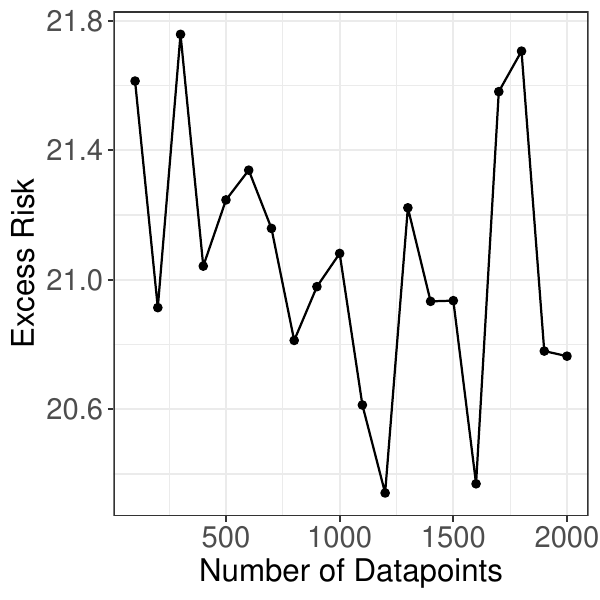}}\label{vanilla}

\caption{Comparison of excess risk of MoMPCA with MCM-PCA and vanilla PCA as discussed in Section \ref{recover}. The vanilla PCA is shown in a different picture for a difference in scale.} \label{fig_recover}
\end{figure*}

In this section, we demonstrate the efficacy of MoMPCA on synthetic and natural datasets. We apply MoMPCA to perform different tasks, including low-rank matrix reconstruction in the presence of outliers, background modeling for video data, and anomaly detection for real data benchmarks. The codes and machine specifications are given at \url{https://github.com/SaptarshiC98/MOMPCA}.

{\textbf{Parameter Selection:} The gradient descent parameter $\eta$ used for the optimization here is typically chosen based on standard methods \cite{ruder2016overview}. In our experiments, we fix $\eta$ at $0.01$. Thus, the only parameter to be selected is the number of partitions $L$. By the relationship $LB=n$, the block size $B$ is chosen automatically once $L$ is fixed. For the PCA to be effective, the block size must be greater than $p$, which imposes an upper bound on $L$ to be $\frac{n}{p}$. On the other hand, to maintain the parity with the theoretical results, we notice in assumption \ref{a3} that $L$ must be greater than $2|\mathcal{O}|$, where $|\mathcal{O}|$ is the number of outliers present in the datapoint. This imposes a lower bound on $L$, and hence in ideal cases, $L$ should be between $2|\mathcal{O}|$ and $\frac{n}{p}$. However, the number of outliers present in a dataset is typically unknown, and $L$ must be greater than $2|\mathcal{O}|$ signifies a particular case where every partition contains exactly one outlier. In practice, since the partitions are done randomly, the probability of that happening is very low. Hence, we may select $L$ to be much less than $2|\mathcal{O}|$. This relaxation in the choice of $L$ leads to the possibility of using the MoMPCA algorithm for datasets where the number of outliers is quite large and $2|\mathcal{O}|$ surpasses the quantity $\frac{n}{p}$. Hence, in practice, we have chosen $L$ to be less than $\frac{n}{p}$, with regard to being sufficiently large so that the median class can be free of outliers.}

\subsection{Simulation Study for Excess Risk Comparison}\label{recover}
To empirically verify the efficacy of MoMPCA, we perform a simulation study to compare against vanilla PCA and the recently proposed MCM-PCA \cite{cardot2017fast}.  We generate the inlying observations of the dataset independently from $p=10$ dimensional multivariate Gaussian distribution with mean $\mathbf{0}$ and covariance matrix $\Sigma=\text{diag}(10,9,\cdot,2,1)$.  Thus, the first two dimensions are most important when projecting the entire data in a two-dimensional subspace.  The outliers, which consist of $10\%$ of datasets, are generated independently from multivariate Gaussian as well, with mean vector $(-20,20,0,0,0,0,0,0,0,0)$ and $I_{10}$ as the covariance matrix. 

We then run all three algorithms on the datase,ts, starting with a total number of datapoints $n=100$ to $n=2000$, on a difference of $100$.  The experiments for each value of $n$ are repeated $500$ times.  We then report the average value of the excess risk for the inliers, which is defined in Eq. \eqref{test_error} takes the form as follows,
\begin{align*}
    \mathfrak{R}(\hat{P}) =\E\left( \|P_0{\bx}\|^2 - \|\hat{P}{\bx}\|_2^2 \right)=\E \, \left({\bx}^\top (P_0-\hat{P}) {\bx}\right)
    &=\E \, \text{Trace}\left({\bx}^\top (P_0-\hat{P}) {\bx}\right)\\
    &= \E \, \text{Trace}\left((P_0-\hat{P}){\bx}{\bx}^\top\right)\\
    &= \text{Trace}\left((P_0-\hat{P})\Sigma\right)
\end{align*}
We calculate the excess risk for each algorithm, where $P_0$ is the actual projection matrix for projecting onto the first two-dimension and $\hat{P}$ is the projection matrix estimated by the respective algorithms.  Then, we plot the average value of test error for each of the $3$ algorithms and compare them in Fig.~\ref{fig_recover}.  We can clearly see that the proposed MoMPCA algorithm performs the best among its peers in terms of test error.  In Fig.~\ref{fig_recover}, we show the excess risk of MoMPCA and MCM-PCA in the same subfigure since both have a similar decay rate, with MoMPCA consistently better.  On the other hand, the vanilla PCA falls apart and has an excess risk much more significant than its peers.  Additionally, one should note that the rate of decrease in the excess risk resembles the risks derived in Corollary \ref{cor1}, validating the theoretical analysis in the process.

\subsection{A Simulation Study on Recovering Low-rank Matrices}\label{low}
To empirically validate the efficacy of MoMPCA, we performed this dimensionality reduction procedure on various simulated datasets.  We have generated the matrix ${\bx}_0$ as a product of two matrices ${\bx}_{1,0}$ and ${\bx}_{2,0}$ of orders $n\times r$ and $r\times p$ respectively where each entry is simulated from the Gaussian distribution, where $r$ is the rank of ${\bx}_0$ and $r< \text{min} \{n, p\}$. To add outliers, we add random noise generated from $Unif(-500,500)$ to randomly selected $\sqrt{n}$ many rows of ${\bx}_0$.

We then run the MoMPCA algorithm along with the peer algorithms on ${\bx}_0$.  We calculate the efficacy of each of the relative reconstruction errors as $\frac{\|{\bx}-{\bx}_0\|}{\|{\bx}_0\|}$, where the norm is taken over those rows that do not contain an outlier. Here $X$ is the projected matrix in the lower $d$-dimensional affine space.  Clearly, a lower score represents a better representation in lower-dimensional space.  Since the true structure of the matrix lies in a lower-dimensional space, the MoMPCA method successfully plots the dataset in a lower-dimensional space despite having $\sqrt{n}$ many outliers due to its robustness property.

We compare our method with the baselines as well as state-of-the-art methods such as Principal Component Pursuit (PCP) \cite{candes2011robust}, Principal Component Pursuit with Features (PCPF) \cite{pmlr-v48-chiang16}, Matrix-Factorization Based Robust Principal Component Analysis (MFRPCA) \cite{chen2018robust} and Robust Weight Learning with Adaptive Neighbors (RWL-AN) \cite{NEURIPS2019_73f104c9}.  The standard protocols used by each of the original papers of the competing algorithms have been implemented in our experiments.  The algorithms providing non-deterministic output have been run $20$ times, and the average values obtained have been reported.  The results demonstrated in table \ref{tab_syn} clearly indicate that out of all the methods, the proposed MoMPCA algorithm works best in terms of the reconstruction error.

\textbf{\textit{Runtime Analysis:}} We run each of the algorithms as demonstrated in section \ref{low} and compare their total runtime. In table \ref{tab_runtime}, the total time in seconds needed to run the algorithms on an M1 Macbook Pro with $8$ GB RAM and $256$ GB storage is provided.  As evident, the first three algorithms, namely PCP \cite{candes2011robust}, PCPF and MFRPCA are all matrix-factorization-based algorithms and hence take a huge amount of time to run as compared to the rest.  The proposed algorithm MoMPCA is run for $50$ iterations since the algorithm mostly converges much before.  We can clearly see from the runtime analysis that the time required to run MoMPCA is almost equivalent to that of MCM-PCA while providing a better performance against the outliers. 

\subsection{Background Modeling in Video}
\label{bmv}
\begin{figure*}
\centering
\subfloat[Original frame]{\includegraphics[width=0.3\textwidth]{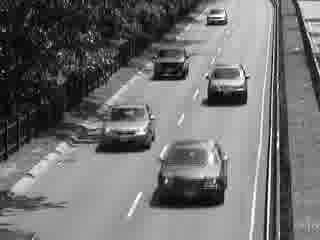}}\label{fig:1a}
\subfloat[Background constructed via low-rank approximation through MoMPCA] {\includegraphics[width=0.3\textwidth]{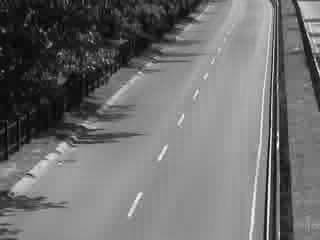}}\label{fig:1b}
\subfloat[Object]{\includegraphics[width=0.3\textwidth]{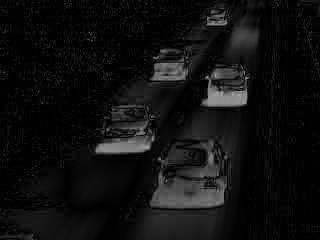}}\label{fig:1c}
\caption{Background modeling through MoMPCA. The background is modeled as a low-dimensional approximation, calculated through MoMPCA as described in section \ref{bmv}.} \label{fig:bmv}
\end{figure*}
\begin{table}\small
    \centering
    \begin{tabular}{c c c c}
    \hline
        Data & \# Dimensions & \# Instances & Anomaly ratio   \\
        \hline
         KDDCUP & 120 & 494,021 &  0.2\\
         Thyroid & 6 & 3,772 & 0.025\\
         Arrhythmia & 274 & 452 & 0.15\\
         \hline
    \end{tabular}
    \caption{Statistics of different anomaly detection benchmarks}
    \label{tab:stat_anomaly}
\end{table}
\begin{table*}
\centering
\resizebox{\textwidth}{!}{
  \begin{tabular}{|l|l l l|l l l|l l l|l|}
    \hline
    \multirow{2}{*}{} &
      \multicolumn{3}{c|}{KDDCUP} &
      \multicolumn{3}{c|}{Thyroid} &
      \multicolumn{3}{c|}{Arrhythmia} &
      \multicolumn{1}{c|}{Average}\\
      \hline
    Method & Precision & Recall & $F_1$ & Precision & Recall & $F_1$ & Precision & Recall & $F_1$ & rank ($F_1)$\\
    \hline
   OC-SVM \cite{chen2001one} & 0.7457 &  0.8523 &  0.7954 & 0.3639 & 0.4239 & 0.3887 & 0.5397 & 0.4082 &  0.4581 & 6.67\\
DSEBM-r \cite{zhang2017joint} & 0.8521 & 0.6472 & 0.7328 & 0.0404 & 0.0403 & 0.0403 & 0.1515 &  0.1513 & 0.1510 & 13.33\\
DSEBM-e \cite{zhang2017joint} & 0.8619 &  0.6446 & 0.7399 & 0.1319 & 0.1319 & 0.1319 & 0.4667 &  0.4565 &  0.4601 & 9.00\\
DCN \cite{yang2017towards} & 0.7696 & 0.7829 & 0.7762 & 0.3319 & 0.3196 &  0.3251 & 0.3758 &  0.3907 &  0.3815 & 8.67\\
GMM-EN \cite{zong2018deep} & 0.1932 & 0.1967 & 0.1949 & 0.0213 & 0.0227 &  0.0220 & 0.3000 &  0.2792 & 0.2886 & 14.33\\
PAE \cite{zong2018deep} & 0.7276 & 0.7397 & 0.7336 & 0.1894 & 0.2062 & 0.1971 & 0.4393 & 0.4437 & 0.4403 & 10.33 \\
E2E-AE \cite{zong2018deep} & 0.0024 &  0.0025 & 0.0024 & 0.1064 & 0.1316 & 0.1176 & 0.4667 & 0.4538 &  0.4591& 11.67\\
PAE-GMM-EM \cite{zong2018deep} & 0.7183 & 0.7311 & 0.7246 & 0.4745 & 0.4538 & 0.4635 & 0.3970 & 0.4168 & 0.4056 & 10.00\\
PAE-GMM \cite{zong2018deep} & 0.7251 & 0.7384 &  0.7317 & 0.4532 & 0.4881 & 0.4688 & 0.4575 & 0.4823 & 0.4684 & 7.17\\
DAGMM-p \cite{zong2018deep} & 0.7579 & 0.7710 & 0.7644 & 0.4723 & 0.4725 & 0.4713 & 0.4909 & 0.4679 & 0.4787 & 4.67\\
DAGMM-NVI \cite{zong2018deep} & 0.9290 & \textbf{0.9447} & 0.9368 & 0.4383 & 0.4587 & 0.4470 & 0.5091 & 0.4892 & 0.4981 & 4.00\\
DAGMM \cite{zong2018deep} & \textbf{0.9297} & 0.9442 & \textbf{0.9369} & 0.4766 & 0.4834 & 0.4782 & 0.4909 & 0.5078 & 0.4983 & 1.67 \\

RSPCA\cite{bian2022robust} & 0.8503 & 0.6587 & 0.7480 & 0.4532 & 0.4881 & 0.4688 & 0.3523 & 0.3765 & 0.3706 & 8.50 \\

SPCA-Barron \cite{dhanaraj2022robust} & 0.7523 & 0.7602 & 0.7657 & 0.2245 & 0.2361 & 0.2398 & 0.4327 & 0.4592 & 0.4516 & 8.33 \\

MoMPCA (Proposed) & 0.8966 & 0.9106 & 0.9035 & \textbf{0.6974} & \textbf{0.5699} & \textbf{0.6272} & \textbf{0.5469} & \textbf{0.5303} & \textbf{0.5385} & \textbf{1.67} \\
    \hline
  \end{tabular}
  }
  \caption{Performance of different peer algorithms for anomaly detection. The average ranks based on the $F_1$ scores are also reported.  The best performances for each data w.r.t.  each metric are boldfaced.}
  \label{anomaly}
\end{table*}
Suppose we are given a video with $f$ many frames.  The objective here is to isolate the moving objects in the video, such as moving cars, pedestrians, etc.  For simplicity, we only consider black and white videos (and thus, there is only a single channel) and let each frame be of size $m \times n$.  We construct $\mathcal{X}=\{{\bx}_1,\dots, {\bx}_{mn}\}\subset \mathbb{R}^f$ by computing a $f \times 1$ vector whose features represent the value of a particular pixel for each frame.  Since consecutive frames are highly correlated, one can expect that the data cloud $\mathcal{X}$ lies near a low-dimensional affine space in $\mathbb{R}^f$.  Thus, the background of the video can be modeled as the projection of $\mathcal{X}$ on this low-dimensional affine space.  To demonstrate the efficacy of modeling the background through MoMPCA, we take the standard video, called ``highway'' (available in the github repository), which captures moving cars on a busy highway.  The video contains $1700$ frames, and each frame is of size $240 \times 320$.  We construct a $76800 \times 1700$ data matrix whose rows denote the data points.  We run the MoMPCA with $L=40$ and $d=5$.  The background is thus given by $\{\bQ{\bx}_1,\dots,\bQ{\bx}_{76800}\}$, where $\bQ$ is the projection matrix, given by Algorithm \ref{algo}. The output background for frame no. 863 is shown in Fig. \ref{fig:bmv} along with the original frame. The object is constructed by taking $\{\|\bx_i-\bQ{\bx}_i\|_1\}_{i=1}^{76800}$. The object conforms to the moving cars in the original frame, and the background only consists of the highway and motionless trees. 

\subsection{An application in Anomaly Detection}
Anomaly or outlier detection is a key problem in machine learning and computer vision.  In this section, we will focus on how to detect outliers through MoMPCA. Suppose $\mathcal{X}=\{{\bx}_1,\dots,{\bx}_n\} \subset \Real^p$ be $n$ data points in the $p$-dimensional real vector space. To detect outliers within the data, we first project $\mathcal{X}$ to a $d$-dimensional affine space via MoMPCA.  We then compute the squared Euclidean distance between the original and projected data points $v_i=\|{\bx}_i - \widehat{\bQ}_{N, L} {\bx}_i\|$ and sort $v_i$'s in ascending order.  We call the $i$-th point an outlier if $v_i$ belongs to the largest $o\%$ of the $\{v_i\}_{i=1}^n$, where $o$ is known beforehand.

For our comparative analysis, we take the KDDCUP, Thyroid, and Arrhythmia datasets, available from the UCI machine learning repository \cite{Dua:2019} and ODDS library \cite{Rayana:2016}.  The details of these datasets are reported in Table \ref{tab:stat_anomaly}.  We compared our method with different state-of-the-art techniques for outlier detection including OC-SVM \cite{chen2001one}; DSEBM-e \& DSEBM-r \cite{zhang2017joint}; DCN \cite{yang2017towards}; GMM-EN, PAE, E2E-AE, PAE-GMM-EM, PAE-GMM, DAGMM-p, DAGMM-NVI, DAGMM \cite{zong2018deep}, RSPCA \cite{bian2022robust} and Stochastic PCA with Barron loss \cite{dhanaraj2022robust}. Many competing algorithms employ a deep neural network to detect anomalies within the data.  To measure the performance of the peer algorithms, we take the average precision, recall, and $F_1$ score between the ground truth and the obtained labeling (inlier/outlier) of the data points.  These measures of accuracy for all three datasets are reported in Table \ref{anomaly}.  The performance indicator values for the peer algorithms are quoted from \cite{zong2018deep}.  It is observed from Table \ref{anomaly} that the MoMPCA is quite competitive against state-of-the-art anomaly detection methods, including even the ones based on deep neural networks.  

\section{Conclusion}
\label{conclusion}

Despite the efficacy, computational simplicity, and ease of visualization of the classical PCA, it often fails to successfully represent the true low-dimensional structure of a dataset in the presence of even a small number of outliers. The traditional approach to formulating robust principal component analysis (PCA) assumes that the data matrix can be decomposed into a low-rank signal component and a noise component. However, this framework is not suitable when dealing with outliers that follow arbitrary distributions and/or are correlated to each other.

This paper proposes an alternative PCA method based on the Median of Means (MoM) estimator to circumvent this problem. The eigen-decomposition trick for PCA cannot be applied in this context, so we use an alternative approach involving projected Adagrad. The proposed MoMPCA, equipped with a computationally simple gradient descent-based optimization procedure, exhibits significant robustness to the presence of outliers. Under minimal and interpretable assumptions, we establish the consistency of MoMPCA in a general separable Hilbert space and find the convergence rate using uniform concentration bounds. The paper's theoretical analysis is carried out with the aid of symmetrization arguments and Rademacher complexities, which, although extensively used in a supervised learning setting, seldom find application in an unsupervised learning scenario.  
The parametric rates for real vector spaces can easily be recovered by taking the Hilbert space to be $\Real^p$. The applicability of our theoretical results to an infinite-dimensional Hilbert space to derive dimension-free bounds with minimal assumptions, is not only novel but also opens up exciting avenues for future research in robust kernel-based methods.

Through practical applications in computer vision, MoMPCA is shown to be compelling relative to even some of the recent deep learning models. The robustness of MoM estimators comes at the expense of slower convergence rates than their ERM counterparts, as demonstrated in the paper. We stress the fact that there is \textit{no wizardry in the median of means estimator} and that the interplay between the partitions and the outliers determines how effective MoM is. A potential future extension of our work could be to make the optimization faster by using adaptive gradient-based optimizers and proving their convergence properties. 
Future research in this direction might render fruitful avenues in improving the ``slow'' ERM rates by establishing fast rates under more restrictive assumptions \cite{boucheron2005theory,wainwright2019high} or finding lower bounds on the approximation error.

\bibliographystyle{unsrtnat}

\begin{thebibliography}{67}
\providecommand{\natexlab}[1]{#1}
\providecommand{\url}[1]{\texttt{#1}}
\expandafter\ifx\csname urlstyle\endcsname\relax
  \providecommand{\doi}[1]{doi: #1}\else
  \providecommand{\doi}{doi: \begingroup \urlstyle{rm}\Url}\fi

\bibitem[Pearson(1901)]{pearson1901liii}
Karl Pearson.
\newblock Liii. on lines and planes of closest fit to systems of points in
  space.
\newblock \emph{The London, Edinburgh, and Dublin Philosophical Magazine and
  Journal of Science}, 2\penalty0 (11):\penalty0 559--572, 1901.

\bibitem[Wold et~al.(1987)Wold, Esbensen, and Geladi]{wold1987principal}
Svante Wold, Kim Esbensen, and Paul Geladi.
\newblock Principal component analysis.
\newblock \emph{Chemometrics and intelligent laboratory systems}, 2\penalty0
  (1-3):\penalty0 37--52, 1987.

\bibitem[Hastie et~al.(2009)Hastie, Tibshirani, and
  Friedman]{hastie2009elements}
Trevor Hastie, Robert Tibshirani, and Jerome Friedman.
\newblock \emph{The elements of statistical learning: data mining, inference,
  and prediction}.
\newblock Springer Science \& Business Media, 2009.

\bibitem[Tipping and Bishop(1999)]{tipping1999probabilistic}
Michael~E Tipping and Christopher~M Bishop.
\newblock Probabilistic principal component analysis.
\newblock \emph{Journal of the Royal Statistical Society: Series B (Statistical
  Methodology)}, 61\penalty0 (3):\penalty0 611--622, 1999.

\bibitem[Sch{\"o}lkopf et~al.(1997)Sch{\"o}lkopf, Smola, and
  M{\"u}ller]{scholkopf1997kernel}
Bernhard Sch{\"o}lkopf, Alexander Smola, and Klaus-Robert M{\"u}ller.
\newblock Kernel principal component analysis.
\newblock In \emph{International conference on artificial neural networks},
  pages 583--588. Springer, 1997.

\bibitem[Zou et~al.(2006)Zou, Hastie, and Tibshirani]{zou2006sparse}
Hui Zou, Trevor Hastie, and Robert Tibshirani.
\newblock Sparse principal component analysis.
\newblock \emph{Journal of computational and graphical statistics}, 15\penalty0
  (2):\penalty0 265--286, 2006.

\bibitem[Cand{\`e}s et~al.(2011)Cand{\`e}s, Li, Ma, and
  Wright]{candes2011robust}
Emmanuel~J Cand{\`e}s, Xiaodong Li, Yi~Ma, and John Wright.
\newblock Robust principal component analysis?
\newblock \emph{Journal of the ACM (JACM)}, 58\penalty0 (3):\penalty0 1--37,
  2011.

\bibitem[Zhang and Tong(2019)]{NEURIPS2019_73f104c9}
Rui Zhang and Hanghang Tong.
\newblock Robust principal component analysis with adaptive neighbors.
\newblock In \emph{Advances in Neural Information Processing Systems},
  volume~32, pages 6961--6969, 2019.

\bibitem[Wright et~al.(2009)Wright, Ganesh, Rao, Peng, and Ma]{wright2009}
John Wright, Arvind Ganesh, Shankar Rao, Yigang Peng, and Yi~Ma.
\newblock Robust principal component analysis: Exact recovery of corrupted
  low-rank matrices via convex optimization.
\newblock In \emph{Advances in Neural Information Processing Systems},
  volume~22, pages 2080--2088, 2009.

\bibitem[Chandrasekaran et~al.(2011)Chandrasekaran, Sanghavi, Parrilo, and
  Willsky]{chandrasekaran2011rank}
Venkat Chandrasekaran, Sujay Sanghavi, Pablo~A Parrilo, and Alan~S Willsky.
\newblock Rank-sparsity incoherence for matrix decomposition.
\newblock \emph{SIAM Journal on Optimization}, 21\penalty0 (2):\penalty0
  572--596, 2011.

\bibitem[Kang et~al.(2015)Kang, Peng, and Cheng]{kang2015robust}
Zhao Kang, Chong Peng, and Qiang Cheng.
\newblock Robust pca via nonconvex rank approximation.
\newblock In \emph{2015 IEEE International Conference on Data Mining}, pages
  211--220. IEEE, 2015.

\bibitem[Chiang et~al.(2016)Chiang, Hsieh, and Dhillon]{pmlr-v48-chiang16}
Kai-Yang Chiang, Cho-Jui Hsieh, and Inderjit Dhillon.
\newblock Robust principal component analysis with side information.
\newblock In \emph{Proceedings of The 33rd International Conference on Machine
  Learning}, volume~48 of \emph{Proceedings of Machine Learning Research},
  pages 2291--2299, New York, New York, USA, 20--22 Jun 2016. PMLR.

\bibitem[Fan and Chow(2020)]{8701558}
Jicong Fan and Tommy W.~S. Chow.
\newblock Exactly robust kernel principal component analysis.
\newblock \emph{IEEE Transactions on Neural Networks and Learning Systems},
  31\penalty0 (3):\penalty0 749--761, 2020.
\newblock \doi{10.1109/TNNLS.2019.2909686}.

\bibitem[Yi et~al.(2020)Yi, He, Jing, Li, Cheung, and Nie]{8818654}
Shuangyan Yi, Zhenyu He, Xiao-Yuan Jing, Yi~Li, Yiu-Ming Cheung, and Feiping
  Nie.
\newblock Adaptive weighted sparse principal component analysis for robust
  unsupervised feature selection.
\newblock \emph{IEEE Transactions on Neural Networks and Learning Systems},
  31\penalty0 (6):\penalty0 2153--2163, 2020.
\newblock \doi{10.1109/TNNLS.2019.2928755}.

\bibitem[Tang and Nehorai(2011)]{tang2011robust}
Gongguo Tang and Arye Nehorai.
\newblock Robust principal component analysis based on low-rank and
  block-sparse matrix decomposition.
\newblock In \emph{2011 45th Annual Conference on Information Sciences and
  Systems}, pages 1--5. IEEE, 2011.

\bibitem[Vaswani et~al.(2018)Vaswani, Bouwmans, Javed, and
  Narayanamurthy]{vaswani2018robust}
Namrata Vaswani, Thierry Bouwmans, Sajid Javed, and Praneeth Narayanamurthy.
\newblock Robust subspace learning: Robust pca, robust subspace tracking, and
  robust subspace recovery.
\newblock \emph{IEEE signal processing magazine}, 35\penalty0 (4):\penalty0
  32--55, 2018.

\bibitem[Wang et~al.(2017)Wang, Gao, Gao, and Nie]{wang2017angle}
Qianqian Wang, Quanxue Gao, Xinbo Gao, and Feiping Nie.
\newblock Angle principal component analysis.
\newblock In \emph{IJCAI}, pages 2936--2942, 2017.

\bibitem[Cardot et~al.(2013)Cardot, C{\'e}nac, and Zitt]{cardot2013efficient}
Herv{\'e} Cardot, Peggy C{\'e}nac, and Pierre-Andr{\'e} Zitt.
\newblock Efficient and fast estimation of the geometric median in hilbert
  spaces with an averaged stochastic gradient algorithm.
\newblock \emph{Bernoulli}, 19\penalty0 (1):\penalty0 18--43, 2013.

\bibitem[Cardot and Godichon-Baggioni(2017)]{cardot2017fast}
Herv{\'e} Cardot and Antoine Godichon-Baggioni.
\newblock Fast estimation of the median covariation matrix with application to
  online robust principal components analysis.
\newblock \emph{Test}, 26\penalty0 (3):\penalty0 461--480, 2017.

\bibitem[Fritz et~al.(2012)Fritz, Filzmoser, and Croux]{fritz2012comparison}
Heinrich Fritz, Peter Filzmoser, and Christophe Croux.
\newblock A comparison of algorithms for the multivariate l 1-median.
\newblock \emph{Computational Statistics}, 27\penalty0 (3):\penalty0 393--410,
  2012.

\bibitem[Cohen et~al.(2016)Cohen, Lee, Miller, Pachocki, and
  Sidford]{cohen2016geometric}
Michael~B Cohen, Yin~Tat Lee, Gary Miller, Jakub Pachocki, and Aaron Sidford.
\newblock Geometric median in nearly linear time.
\newblock In \emph{Proceedings of the forty-eighth annual ACM symposium on
  Theory of Computing}, pages 9--21, 2016.

\bibitem[Vapnik(2013)]{vapnik2013nature}
Vladimir Vapnik.
\newblock \emph{The nature of statistical learning theory}.
\newblock Springer science \& business media, 2013.

\bibitem[Lugosi et~al.(2019)Lugosi, Mendelson,
  et~al.]{lugosi2019regularization}
G{\'a}bor Lugosi, Shahar Mendelson, et~al.
\newblock Regularization, sparse recovery, and median-of-means tournaments.
\newblock \emph{Bernoulli}, 25\penalty0 (3):\penalty0 2075--2106, 2019.

\bibitem[Lecu{\'e} et~al.(2020{\natexlab{a}})Lecu{\'e}, Lerasle,
  et~al.]{lecue2020robust}
Guillaume Lecu{\'e}, Matthieu Lerasle, et~al.
\newblock Robust machine learning by median-of-means: theory and practice.
\newblock \emph{Annals of Statistics}, 48\penalty0 (2):\penalty0 906--931,
  2020{\natexlab{a}}.

\bibitem[Bartlett et~al.(2002)Bartlett, Boucheron, and
  Lugosi]{bartlett2002model}
Peter~L Bartlett, St{\'e}phane Boucheron, and G{\'a}bor Lugosi.
\newblock Model selection and error estimation.
\newblock \emph{Machine Learning}, 48\penalty0 (1):\penalty0 85--113, 2002.

\bibitem[Lecu{\'e} et~al.(2020{\natexlab{b}})Lecu{\'e}, Lerasle, and
  Mathieu]{lecue2020class}
Guillaume Lecu{\'e}, Matthieu Lerasle, and Timloth{\'e}e Mathieu.
\newblock Robust classification via mom minimization.
\newblock \emph{Machine Learning}, 109\penalty0 (8):\penalty0 1635--1665,
  2020{\natexlab{b}}.

\bibitem[Lerasle(2019)]{lerasle2019lecture}
Matthieu Lerasle.
\newblock Lecture notes: Selected topics on robust statistical learning theory.
\newblock \emph{arXiv preprint arXiv:1908.10761}, 2019.

\bibitem[Laforgue et~al.(2019)Laforgue, Cl{\'e}men{\c{c}}on, and
  Bertail]{laforgue2019medians}
Pierre Laforgue, St{\'e}phan Cl{\'e}men{\c{c}}on, and Patrice Bertail.
\newblock On medians of (randomized) pairwise means.
\newblock In \emph{International Conference on Machine Learning}, pages
  1272--1281. PMLR, 2019.

\bibitem[Mathieu and Minsker(2021)]{mathieu2021excess}
Timoth{\'e}e Mathieu and Stanislav Minsker.
\newblock Excess risk bounds in robust empirical risk minimization.
\newblock \emph{Information and Inference: A Journal of the IMA}, 2021.

\bibitem[Bubeck et~al.(2013)Bubeck, Cesa-Bianchi, and
  Lugosi]{bubeck2013bandits}
S{\'e}bastien Bubeck, Nicolo Cesa-Bianchi, and G{\'a}bor Lugosi.
\newblock Bandits with heavy tail.
\newblock \emph{IEEE Transactions on Information Theory}, 59\penalty0
  (11):\penalty0 7711--7717, 2013.

\bibitem[Minsker(2018)]{minsker2018uniform}
Stanislav Minsker.
\newblock Uniform bounds for robust mean estimators.
\newblock \emph{arXiv preprint arXiv:1812.03523}, 2018.

\bibitem[Klochkov et~al.(2020)Klochkov, Kroshnin, and
  Zhivotovskiy]{klochkov2020robust}
Yegor Klochkov, Alexey Kroshnin, and Nikita Zhivotovskiy.
\newblock Robust $ k $-means clustering for distributions with two moments.
\newblock \emph{arXiv preprint arXiv:2002.02339}, 2020.

\bibitem[Brunet-Saumard et~al.(2020)Brunet-Saumard, Genetay, and
  Saumard]{brunet2020k}
Camille Brunet-Saumard, Edouard Genetay, and Adrien Saumard.
\newblock K-bmom: a robust lloyd-type clustering algorithm based on bootstrap
  median-of-means.
\newblock \emph{arXiv preprint arXiv:2002.03899}, 2020.

\bibitem[Paul et~al.(2021)Paul, Chakraborty, Das, and Xu]{paul2021uniform}
Debolina Paul, Saptarshi Chakraborty, Swagatam Das, and Jason Xu.
\newblock Uniform concentration bounds toward a unified framework for robust
  clustering.
\newblock \emph{Advances in Neural Information Processing Systems},
  34:\penalty0 8307--8319, 2021.

\bibitem[Lecu{\'e} et~al.(2020{\natexlab{c}})Lecu{\'e}, Lerasle, and
  Mathieu]{lecue2020robust1}
Guillaume Lecu{\'e}, Matthieu Lerasle, and Timloth{\'e}e Mathieu.
\newblock Robust classification via mom minimization.
\newblock \emph{Machine Learning}, 109\penalty0 (8):\penalty0 1635--1665,
  2020{\natexlab{c}}.

\bibitem[Staerman et~al.(2021)Staerman, Laforgue, Mozharovskyi, and d'Alch{\'e}
  Buc]{pmlr-v130-staerman21a}
Guillaume Staerman, Pierre Laforgue, Pavlo Mozharovskyi, and Florence
  d'Alch{\'e} Buc.
\newblock When ot meets mom: Robust estimation of wasserstein distance.
\newblock In Arindam Banerjee and Kenji Fukumizu, editors, \emph{Proceedings of
  The 24th International Conference on Artificial Intelligence and Statistics},
  volume 130 of \emph{Proceedings of Machine Learning Research}, pages
  136--144. PMLR, 13--15 Apr 2021.
\newblock URL \url{http://proceedings.mlr.press/v130/staerman21a.html}.

\bibitem[Bartlett and Mendelson(2002)]{bartlett2002rademacher}
Peter~L Bartlett and Shahar Mendelson.
\newblock Rademacher and gaussian complexities: Risk bounds and structural
  results.
\newblock \emph{Journal of Machine Learning Research}, 3\penalty0
  (Nov):\penalty0 463--482, 2002.

\bibitem[Devroye et~al.(2013)Devroye, Gy{\"o}rfi, and
  Lugosi]{devroye2013probabilistic}
Luc Devroye, L{\'a}szl{\'o} Gy{\"o}rfi, and G{\'a}bor Lugosi.
\newblock \emph{A probabilistic theory of pattern recognition}, volume~31.
\newblock Springer Science \& Business Media, 2013.

\bibitem[Lecu{\'e} et~al.(2020{\natexlab{d}})Lecu{\'e}, Lerasle, and
  Mathieu]{lecue2020robustcl}
Guillaume Lecu{\'e}, Matthieu Lerasle, and Timloth{\'e}e Mathieu.
\newblock Robust classification via mom minimization.
\newblock \emph{Machine Learning}, 109\penalty0 (8):\penalty0 1635--1665,
  2020{\natexlab{d}}.

\bibitem[Zhang et~al.(2015)Zhang, Zhou, and Liang]{zhang2015analysis}
Huishuai Zhang, Yi~Zhou, and Yingbin Liang.
\newblock Analysis of robust pca via local incoherence.
\newblock \emph{Advances in Neural Information Processing Systems}, 28, 2015.

\bibitem[Liu et~al.(2014)Liu, Feng, and Qiao]{liu2014scatter}
Shenglan Liu, Lin Feng, and Hong Qiao.
\newblock Scatter balance: An angle-based supervised dimensionality reduction.
\newblock \emph{IEEE transactions on neural networks and learning systems},
  26\penalty0 (2):\penalty0 277--289, 2014.

\bibitem[Liu and Yu(2020)]{liu2020angular}
Shenglan Liu and Yang Yu.
\newblock Angular embedding: A new angular robust principal component analysis.
\newblock \emph{arXiv preprint arXiv:2011.11013}, 2020.

\bibitem[Hauberg et~al.(2014)Hauberg, Feragen, and Black]{hauberg2014grassmann}
Soren Hauberg, Aasa Feragen, and Michael~J Black.
\newblock Grassmann averages for scalable robust pca.
\newblock In \emph{Proceedings of the IEEE Conference on Computer Vision and
  Pattern Recognition}, pages 3810--3817, 2014.

\bibitem[Hauberg et~al.(2015)Hauberg, Feragen, Enficiaud, and
  Black]{hauberg2015scalable}
S{\o}ren Hauberg, Aasa Feragen, Raffi Enficiaud, and Michael~J Black.
\newblock Scalable robust principal component analysis using grassmann
  averages.
\newblock \emph{IEEE transactions on pattern analysis and machine
  intelligence}, 38\penalty0 (11):\penalty0 2298--2311, 2015.

\bibitem[Bouwmans et~al.(2018)Bouwmans, Javed, Zhang, Lin, and
  Otazo]{bouwmans2018applications}
Thierry Bouwmans, Sajid Javed, Hongyang Zhang, Zhouchen Lin, and Ricardo Otazo.
\newblock On the applications of robust pca in image and video processing.
\newblock \emph{Proceedings of the IEEE}, 106\penalty0 (8):\penalty0
  1427--1457, 2018.

\bibitem[Gao et~al.(2017)Gao, Ma, Liu, Gao, and Nie]{gao2017angle}
Quanxue Gao, Lan Ma, Yang Liu, Xinbo Gao, and Feiping Nie.
\newblock Angle 2dpca: A new formulation for 2dpca.
\newblock \emph{IEEE transactions on cybernetics}, 48\penalty0 (5):\penalty0
  1672--1678, 2017.

\bibitem[Gao et~al.(2020)Gao, Zhang, Xia, Xie, Gao, and Tao]{gao2020enhanced}
Quanxue Gao, Pu~Zhang, Wei Xia, Deyan Xie, Xinbo Gao, and Dacheng Tao.
\newblock Enhanced tensor rpca and its application.
\newblock \emph{IEEE transactions on pattern analysis and machine
  intelligence}, 43\penalty0 (6):\penalty0 2133--2140, 2020.

\bibitem[Liao et~al.(2018)Liao, Li, Liu, Gao, and Gao]{liao2018robust}
Shuangli Liao, Jin Li, Yang Liu, Quanxue Gao, and Xinbo Gao.
\newblock Robust formulation for pca: Avoiding mean calculation with l 2,
  p-norm maximization.
\newblock In \emph{Thirty-Second AAAI Conference on Artificial Intelligence},
  2018.

\bibitem[Lecu{\'e} and Lerasle(2019)]{lecue2019learning}
Guillaume Lecu{\'e} and Matthieu Lerasle.
\newblock Learning from mom’s principles: Le cam’s approach.
\newblock \emph{Stochastic Processes and Their Applications}, 129\penalty0
  (11):\penalty0 4385--4410, 2019.

\bibitem[Rodriguez and Valdora(2019)]{rodriguez2019breakdown}
Daniela Rodriguez and Marina Valdora.
\newblock The breakdown point of the median of means tournament.
\newblock \emph{Statistics \& Probability Letters}, 153:\penalty0 108--112,
  2019.

\bibitem[Arora et~al.(2012)Arora, Cotter, Livescu, and
  Srebro]{arora2012stochastic}
Raman Arora, Andrew Cotter, Karen Livescu, and Nathan Srebro.
\newblock Stochastic optimization for pca and pls.
\newblock In \emph{2012 50th Annual Allerton Conference on Communication,
  Control, and Computing (Allerton)}, pages 861--868. IEEE, 2012.

\bibitem[Shalev-Shwartz and Ben-David(2014)]{shalev2014understanding}
Shai Shalev-Shwartz and Shai Ben-David.
\newblock \emph{Understanding machine learning: From theory to algorithms}.
\newblock Cambridge university press, 2014.

\bibitem[Athreya and Lahiri(2006)]{athreya2006measure}
Krishna~B Athreya and Soumendra~N Lahiri.
\newblock \emph{Measure theory and probability theory}.
\newblock Springer Science \& Business Media, 2006.

\bibitem[Pollard(1981)]{pollard1981}
David Pollard.
\newblock Strong consistency of $k$-means clustering.
\newblock \emph{Ann. Statist.}, 9\penalty0 (1):\penalty0 135--140, 01 1981.
\newblock \doi{10.1214/aos/1176345339}.
\newblock URL \url{https://doi.org/10.1214/aos/1176345339}.

\bibitem[Chakraborty et~al.(2020)Chakraborty, Paul, Das, and
  Xu]{chakraborty2020entropy}
Saptarshi Chakraborty, Debolina Paul, Swagatam Das, and Jason Xu.
\newblock Entropy weighted power k-means clustering.
\newblock In \emph{International Conference on Artificial Intelligence and
  Statistics}, pages 691--701. PMLR, 2020.

\bibitem[Chen and Zhou(2018)]{chen2018robust}
Yongyong Chen and Yicong Zhou.
\newblock Robust principal component analysis with matrix factorization.
\newblock In \emph{2018 IEEE International Conference on Acoustics, Speech and
  Signal Processing (ICASSP)}, pages 2411--2415. IEEE, 2018.

\bibitem[Ruder(2016)]{ruder2016overview}
Sebastian Ruder.
\newblock An overview of gradient descent optimization algorithms.
\newblock \emph{arXiv preprint arXiv:1609.04747}, 2016.

\bibitem[Chen et~al.(2001)Chen, Zhou, and Huang]{chen2001one}
Yunqiang Chen, Xiang~Sean Zhou, and Thomas~S Huang.
\newblock One-class svm for learning in image retrieval.
\newblock In \emph{Proceedings 2001 International Conference on Image
  Processing (Cat. No. 01CH37205)}, volume~1, pages 34--37. IEEE, 2001.

\bibitem[Zhang and Woodland(2017)]{zhang2017joint}
Chao Zhang and Philip~C Woodland.
\newblock Joint optimisation of tandem systems using gaussian mixture density
  neural network discriminative sequence training.
\newblock In \emph{2017 IEEE International Conference on Acoustics, Speech and
  Signal Processing (ICASSP)}, pages 5015--5019. IEEE, 2017.

\bibitem[Yang et~al.(2017)Yang, Fu, Sidiropoulos, and Hong]{yang2017towards}
Bo~Yang, Xiao Fu, Nicholas~D Sidiropoulos, and Mingyi Hong.
\newblock Towards k-means-friendly spaces: Simultaneous deep learning and
  clustering.
\newblock In \emph{international conference on machine learning}, pages
  3861--3870. PMLR, 2017.

\bibitem[Zong et~al.(2018)Zong, Song, Min, Cheng, Lumezanu, Cho, and
  Chen]{zong2018deep}
Bo~Zong, Qi~Song, Martin~Renqiang Min, Wei Cheng, Cristian Lumezanu, Daeki Cho,
  and Haifeng Chen.
\newblock Deep autoencoding gaussian mixture model for unsupervised anomaly
  detection.
\newblock In \emph{International Conference on Learning Representations}, 2018.

\bibitem[Bian et~al.(2022)Bian, Zhao, Nie, Wang, and Li]{bian2022robust}
Jintang Bian, Dandan Zhao, Feiping Nie, Rong Wang, and Xuelong Li.
\newblock Robust and sparse principal component analysis with adaptive loss
  minimization for feature selection.
\newblock \emph{IEEE Transactions on Neural Networks and Learning Systems},
  2022.

\bibitem[Dhanaraj and Markopoulos(2022)]{dhanaraj2022robust}
Mayur Dhanaraj and Panos~P Markopoulos.
\newblock Robust stochastic principal component analysis via barron loss.
\newblock In \emph{2022 56th Asilomar Conference on Signals, Systems, and
  Computers}, pages 1286--1290. IEEE, 2022.

\bibitem[Dua and Graff(2017)]{Dua:2019}
Dheeru Dua and Casey Graff.
\newblock {UCI} machine learning repository, 2017.
\newblock URL \url{http://archive.ics.uci.edu/ml}.

\bibitem[Rayana(2016)]{Rayana:2016}
Shebuti Rayana.
\newblock {ODDS} library, 2016.
\newblock URL \url{http://odds.cs.stonybrook.edu}.

\bibitem[Boucheron et~al.(2005)Boucheron, Bousquet, and
  Lugosi]{boucheron2005theory}
St{\'e}phane Boucheron, Olivier Bousquet, and G{\'a}bor Lugosi.
\newblock Theory of classification: A survey of some recent advances.
\newblock \emph{ESAIM: Probability and Statistics}, 9:\penalty0 323--375, 2005.

\bibitem[Wainwright(2019)]{wainwright2019high}
Martin~J Wainwright.
\newblock \emph{High-dimensional statistics: A non-asymptotic viewpoint},
  volume~48.
\newblock Cambridge University Press, 2019.

\end{thebibliography}

\end{document}